\documentclass[lettersize,journal]{IEEEtran}

\usepackage{amsmath}
\usepackage{amssymb}
\usepackage{amsfonts}
\usepackage{mathrsfs}
\usepackage{mathtools}
\usepackage{amsthm}
\usepackage{algorithm}
\usepackage{algorithmic}
\usepackage{array}
\usepackage{booktabs}
\usepackage{makecell}
\usepackage{graphicx}
\usepackage{textcomp}
\usepackage{stfloats}
\usepackage{url}
\usepackage{cite}
\usepackage{pifont}
\usepackage[table,xcdraw]{xcolor}
\usepackage{multirow}
\usepackage{adjustbox}
\usepackage{diagbox}
\usepackage{microtype}
\usepackage{hyperref}
\usepackage[capitalize,noabbrev]{cleveref}

\theoremstyle{plain}
\newtheorem{theorem}{Theorem}[section]

\theoremstyle{definition}

\theoremstyle{remark}

\newcommand{\thickhline}{\noalign{\hrule height 1pt}}

\begin{document}

\title{LION: A Clifford Neural Paradigm for Multimodal-Attributed Graph Learning}

\author{
    \IEEEauthorblockN{
    Xunkai Li$^{\dagger,*}$,
    Zekai Chen$^{\dagger,*}$,
    Zhengyu Wu$^\dagger$,
    Henan Sun$^\dagger$,
    Daohan Su$^\dagger$,
    Guang Zeng$^\#$,
    Hongchao Qin$^\dagger$, \\
    Rong-Hua Li$^\dagger$,
    Guoren Wang$^\dagger$}
    
    \IEEEauthorblockA{
    $^\dagger$ Beijing Institute of Technology, Beijing, China; $^*$ Equal contribution.}
    
    \IEEEauthorblockA{
    $^\#$ Ant Group, Hangzhou, China}
    
    \IEEEauthorblockA{
    cs.xunkai.li@gmail.com,
    zackchen02@163.com,
    jeremywzy96@outlook.com,
    magneto0617@foxmail.com,
    dhsu@bit.edu.cn,
    zengguang\_77@qq.com,
    qhc.neu@gmail.com,
    lironghuabit@126.com,
    wanggrbit@126.com}
}

\markboth{IEEE Transactions on Knowledge and Data Engineering}%
{Li \MakeLowercase{\textit{et al.}}: LION: A Clifford Neural Paradigm for Multimodal-Attributed Graph Learning}

\maketitle

\begin{abstract}
    Recently, the rapid advancement of multimodal domains has driven a data-centric paradigm shift in graph ML, transitioning from text-attributed to multimodal-attributed graphs.
    This advancement significantly enhances data representation and expands the scope of graph downstream tasks, such as modality-oriented tasks, thereby improving the practical utility of graph ML.
    Despite its promise, limitations exist in the current neural paradigms:
    (1) Neglect Context in Modality Alignment: 
    Most existing methods adopt topology-constrained or modality-specific operators as tokenizers.
    These aligners inevitably neglect graph context and inhibit modality interaction, resulting in suboptimal alignment.
    (2) Lack of Adaptation in Modality Fusion: 
    Most existing methods are simple adaptations for 2-modality graphs and fail to adequately exploit aligned tokens equipped with topology priors during fusion, leading to poor generalizability and performance degradation.

    To address the above issues, we propose LION (c\underline{LI}ff\underline{O}rd \underline{N}eural paradigm) based on the Clifford algebra and decoupled graph neural paradigm (i.e., propagation-then-aggregation) to implement alignment-then-fusion in multimodal-attributed graphs. 
    Specifically, we first construct a modality-aware geometric manifold grounded in Clifford algebra.
    This geometric-induced high-order graph propagation efficiently achieves modality interaction, facilitating modality alignment.
    Then, based on the topology-aware Clifford components of aligned tokens, we propose adaptive holographic aggregation. This module integrates component-wise energy and propagation-scale information with learnable parameters to improve modality fusion.
    Extensive experiments on 9 text-image MAG datasets demonstrate that LION significantly outperforms SOTA baselines across 3 graph and 3 modality downstream tasks.
\end{abstract}


\section{Introduction}
\label{sec: Introduction}
    With the proliferation of multimodal data (e.g., images and texts) and vision-language models~\cite{radford2021_clip,zhang2022_facebook_opt,rombach2022_stable_diffusion}, multimodal-attributed graphs (MAGs) have recently emerged as a research frontier in the graph ML community. 
    In MAG, nodes denote real-world entities described by multimodal data, which substantially expands the semantic dimensionality of node features and elevates the upper bound of representation learning capability~\cite{yan2025magb,zhu2025mm_graph,liu2025graph_mllm,wang2025mg_llm}.
    This means that MAG not only offers immense potential for enhancing traditional graph-based tasks (e.g., node classification and link prediction)~\cite{kipf2016gcn,hamilton2017graphsage,velivckovic2017gat} but also broadens the task spectrum to encompass modality-oriented retrieval and generation~\cite{ning2025graph4mm,fang2025graphgpt_o,jin2024instructg2i}, thereby enhancing practical utility.
    Consequently, developing effective MAG neural paradigms (i.e., MAGNNs) has become an urgent necessity~\cite{hu2021mmgcn,tao2020mgat,li2025c_mag,cai2024omg_nas}.
    Despite the remarkable progress made by recent studies, some inherent limitations still exist. 
    Our in-depth analysis is as follows.

    \textbf{{Limitation: Neglect MAG-Context in Modality Alignment.}}
    We argue that most existing methods inherit unnecessary entity isolation, at the node and modality levels, from conventional graph-agnostic modality alignment strategies.
    Specifically, (1) some methods employ topology-constrained aligners~\cite{fan2025mlaga,fang2025graphgpt_o,ning2025graph4mm,jin2024instructg2i}. 
    While they aim to mitigate the receptive field limitations of the target nodes (i.e., single-entity) in MAG by incorporating broader node contexts, their scope remains restricted to 1-hop neighbors.
    Consequently, they still fail to capture the long-range dependencies, which have been proven to be highly beneficial in graphs~\cite{2019appnp,chien2021gprgnn,frasca2020sign,zhang2021ndls,li2024atp,liu2025sigma}.
    (2) While other methods explicitly leverage graph topology to facilitate global alignment, they remain restricted to modality-specific scenarios and rely on the same topology to apply identical alignment across different modalities~\cite{hu2025mig_gt,hu2025ntsformer,guo2025DMGC,he2025unigraph2}. 
    This coarse-grained alignment overlooks beneficial intra- and inter-modal interactions, thereby negatively impacting modality fusion.

    \textbf{Key Insight}.
    It is imperative to leverage modality-aware, high-quality topology to dissolve entity semantic boundaries in MAGs, thereby fostering beneficial, profound modality interactions and understanding for better alignment.

    \textbf{Solution and Evaluation}.
    To this end, we propose Clifford Geometric Propagation (CGP) to achieve topology-aware modality interaction for better alignment.
    Specifically, we first construct a MAG-specific geometric manifold where structural relationships are modeled as spatial rotations, and distinct modalities are represented as orthogonal geometric base vectors rather than modality-specific scale vectors.
    We then introduce a topology-aware geometric potential to characterize local semantic compatibility and regulate the contribution of neighboring signals during geometric transport.
    Together, the geometric potential and spatial rotor induce curvature-adaptive high-order graph propagation, enabling intra- and inter-modality interactions to facilitate modality alignment.
    To evaluate CGP, we illustrate the performance gains in Fig.~\ref{fig: empirical}(a) by employing CGP as a plug-and-play module to replace existing modality aligners.
    Despite several deeply coupled architectures that prevent broader plug-and-play evaluation, the consistent gains across these four representative cases support the effectiveness of CGP.

    \textbf{{Limitation: Lack of MAG-Adaptation in Modality Fusion.}}
    Given that enriched context, MAG significantly increases both the quantity and information density of aligned tokens, thereby presenting unique challenges.
    Specifically, (1) some methods directly employ Q-Former~\cite{li2023blip} to handle input scenarios involving topology-driven tokens (i.e., 1-hop neighbor tokens) by adaptive queries~\cite{fan2025mlaga,fang2025graphgpt_o,ning2025graph4mm,jin2024instructg2i}.
    Despite their simplicity and intuitiveness, these methods fail to explicitly leverage semantic-aware topology priors within the aligned tokens.
    (2) Although other methods attempt to enhance fusion by partitioning tokens into multiple topology- and modality-specific views and achieving fine-grained integration, such approaches offer only limited exploration of suboptimal aligned tokens and fail to deeply capture their intricate dependencies.~\cite{hu2025mig_gt,hu2025ntsformer,guo2025DMGC,he2025unigraph2,shen2025sr_gm}.
    
    \textbf{Key Insight}.
    It is imperative to reveal the intricate dependency among aligned tokens and design a comprehensive fusion mechanism to unleash their representation potential.

    \textbf{Solution and Evaluation}.
    After obtaining the propagated features in the geometric manifold (i.e., aligned tokens) via the CGP, we clarify that they naturally possess geometric grade properties, which encode the topology and modality insights.
    Based on this, we propose Adaptive Holographic Aggregation (AHA), which employs learnable parameters to explicitly model the energy and scale inherent in these geometric grades.
    This strategy functions as a dynamic filter, capturing the most relevant topology and modality insights for task-adaptive fusion representation. 
    To substantiate our claims, we present comprehensive results in Fig.~\ref{fig: empirical}(b), where LION demonstrates superior convergence and performance compared to the SOTA baselines.
    Furthermore, the complete LION (i.e., LION w/ AHA) significantly outperforms the LION w/o AHA variant, thereby validating the advantages of modality fusion in the holographic perspective.

\begin{figure}[t]
  \centering
  \includegraphics[width=\linewidth]{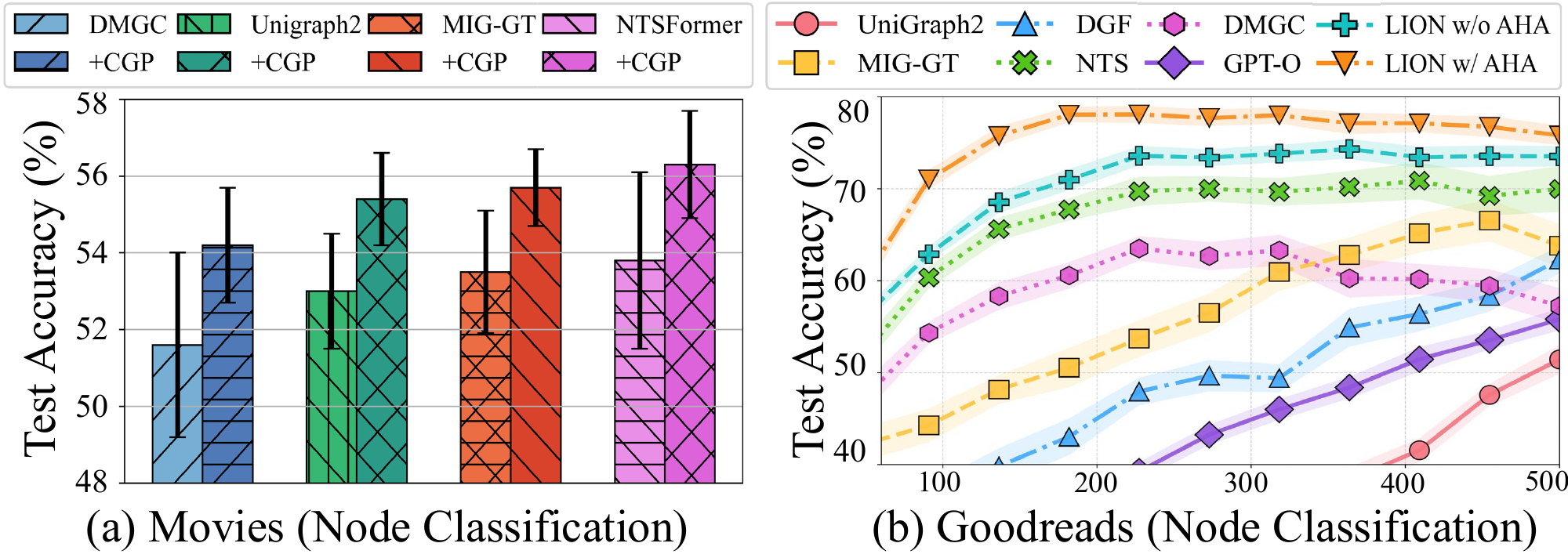}
  \caption{Empirical evaluations of CGP (a) and AHA (b), where the x-axis for AHA (b) is training time in seconds.}
  \label{fig: empirical}
\end{figure}

    \textbf{Contributions}.
    (1) \underline{\textit{New Perspective}}.
    We reveal the limitations of modality alignment and fusion in MAGs and introduce a mathematical framework grounded in Clifford algebra to address them.
    (2) \underline{\textit{Innovative Method}}.
    We propose LION (c\underline{LI}ff\underline{O}rd \underline{N}eural paradigm) based on this mathematical framework, which achieves alignment-then-fusion by propagation(CGP)-then-aggregation(AHA). 
    Specifically, CGP first constructs a MAG-specific geometric manifold to enable curvature-adaptive high-order graph propagation, achieving efficient modality interaction for better alignment. 
    Subsequently, AHA captures the energy of topology-aware Clifford components and their scale-dependent representations via learnable weights to facilitate modality fusion.
    (3) \underline{\textit{SOTA Performance}}.
    Evaluations from 6 domains demonstrate that LION achieves average improvements of 5.24\% and 7.68\% over SOTA baselines in graph and modality tasks.

\section{Preliminaries}

\subsection{Notations and Problem Formulation}
    MAG is defined as $\mathcal{G}=(\mathcal{V}, \mathcal{E}, \{\mathbf{X}^{(m)}\}_{m\in\mathcal{M}})$ where $\mathcal{V}$ is the set of nodes, $\mathcal{E}$ is the set of edges, and $\mathcal{M}$ is the set of available modalities (e.g., texts and images). 
    For each node $v_i$ and modality $m$, modality-specific attribute vector ${x}_i^{(m)} \in \mathbb{R}^{d_m}$ constitutes complete $\mathbf{X} \in \mathbb{R}^{N \times d}$.
    The modality-shared topology is described by the same adjacency matrix $\mathbf{A}$, and $\mathbf{D}$ is the corresponding diagonal degree matrix.
    We further define the symmetric normalized graph Laplacian as $\tilde{\mathbf{L}} =\mathbf{I}-\tilde{\mathbf{D}}^{-1/2}\tilde{\mathbf{A}}\tilde{\mathbf{D}}^{-1/2}$ with self-loop ($\tilde{\mathbf{A}}=\mathbf{I}+\mathbf{A}$), which captures structural smoothness in graph signal processing~\cite{bruna2014spectral,2016Convolutional_ChebNet}.
    This formulation separates modality-specific attributes from the modality-shared topology, allowing us to analyze alignment and fusion as two coupled but distinct operations.
    It also clarifies the role of graph structure: topology is not merely an auxiliary input, but a shared geometric prior that regulates how multimodal signals interact across nodes.
    Accordingly, the learned representation should preserve both modality-specific semantics and topology-induced relational context.

    \textbf{Graph Tasks}.
    (1) Node Classification: predict the specific label class of unlabeled nodes~\cite{Yang16cora,shchur2018amazon_datasets,hu2020ogb}; 
    (2) Link Prediction: predict if $(u,v)\in\mathcal{E}$ exists in the edge sets~\cite{Zhang18link_prediction1,cai2021link_prediction2}; 
    (3) Node Clustering: partition all nodes into clusters using off-the-shelf clustering algorithms applied to the learned representations~\cite{blondel2008louvain,karypis1998metis}.
    These graph-centric tasks examine whether learned representations preserve structural proximity, class-level semantics, and community-level organization.
    They therefore provide a direct test of whether multimodal attributes can enhance classical graph learning beyond purely topology-driven modeling.

    \textbf{Multimodal Tasks}. 
    (1) Modality Retrieval: given a query in a specific modality (e.g., text), retrieve the corresponding cross-modal representation;
    (2) Text Generation: generate text responses conditioned on the target node, instructions, and graph contexts; 
    (3) Image Generation: generate visual responses using diffusion models conditioned on the target node's textual descriptions and the graph contexts.
    Compared with graph-centric tasks, these modality-oriented tasks require the model to preserve fine-grained cross-modal semantics while respecting graph-induced dependencies.
    This dual requirement makes MAG learning fundamentally different from applying a graph encoder after independent multimodal feature extraction.

\subsection{MAG Neural Paradigm}
    \textbf{Conventional Methods.} 
    These approaches typically adapt previous multimodal architectures to MAGs by treating structural information as auxiliary tokens. 
    The core idea is linearizing neighborhoods into sequences~\cite{ning2025graph4mm,fang2025graphgpt_o} or projecting topology priors into soft tokens~\cite{fan2025mlaga,jin2024instructg2i}. 
    While leveraging pre-trained knowledge, these methods rely on concatenation or rigid alignment, lacking a unified integration where topology and modality are unified in a continuous space.
    As a result, the graph often functions as an external context provider rather than a principle that governs modality interaction.
    Such a design is effective for injecting local context, but it may underuse topology when long-range or cross-modality dependencies are required.

    \textbf{Graph-enhanced Methods.} 
    These approaches explicitly design specialized model architectures for MAGs-such as refined message-passing~\cite{gilmer2017neural,hu2025mig_gt,hu2025ntsformer}, spectral graph filtering~\cite{bruna2014spectral,2016Convolutional_ChebNet,shen2025sr_gm,guo2025DMGC}, and structure-conditioned modules~\cite{he2025unigraph2,zheng2025dgf}-to capture topology and modality insights. 
    However, these approaches often treat multimodal attributes and graph topology as separate semantic spaces, relying on well-designed but suboptimal alignment and fusion components.
    This motivates a paradigm that represents modalities and topology within a shared geometric space, so that alignment and fusion can be derived from the same underlying operators.

\section{Methodology}
\label{sec: Methodology}
\subsection{Clifford Algebra and Geometric Manifold}
\label{sec: Clifford Algebra and Geometric Manifold}
    In this paper, we construct a mathematical framework specifically tailored for MAGs based on Clifford algebra~\cite{hestenes2012clifford_1,lounesto2001clifford_2,dorst2009clifford_3}, which extends the modality-specific graph operations to modality-aware high-dimensional geometric manifold spaces.
    Let $\mathcal{C}l_{n}$ denote the Clifford algebra defined over an $n$-dimensional space. 
    By equipping each node with a local tangent space isomorphic to $\mathcal{C}l_{n}$ and modeling edges as geodesic connections between them, we conceptualize the MAG as a discrete geometric manifold.
    The fundamental graph operation is the geometric product, which unifies spatial rotations (topology) with orthogonal geometric base vectors (modality).
    
    Specifically, for any two connected nodes $u, v$ with attribute vectors $x_u, x_v$, their topology-aware modality interaction (i.e., transform modality-aware geometric base vectors via spatial rotations in the geometric manifolds) is formalized as $x_u x_v = x_u \cdot x_v + x_u \wedge x_v$.
    This geometric product explicitly decomposes the multimodal interaction:
    (1) The symmetric inner product (${x}_u \cdot {x}_v$) induces a modality-specific projection, facilitating intra-modality interaction.
    (2) The antisymmetric outer product ($x_u \wedge x_v$) instantiates a bi-vector plane that encodes the topology curvature induced by cross-modality discrepancies, facilitating inter-modality interaction.
    This geometric manifold, spanned by Clifford algebra, provides a mathematical framework for topology-aware intra- and inter-modality interaction. 
    Based on this, it enables modality alignment through topology-driven, curvature-induced spatial rotations (i.e., graph propagation) and facilitates modality fusion by capturing the geometric grade properties of aligned tokens (i.e., message aggregation).

\subsection{Clifford Geometric Propagation}
\label{sec: Clifford Geometric Propagation}
\textbf{Motivation.}
    In order to achieve efficient modality alignment via the CGP (i.e., modality-aware, curvature-adaptive high-order graph propagation within the geometric manifold), it is essential to capture topology insights from the curved manifold to facilitate multimodal interactions.
    Consequently, we introduce parallel transport, which aligns the distinct semantic spaces of multiple modalities within connected nodes by rotating modality-aware orthogonal geometric basis vectors along the topology curvature.
    This high-dimensional transformation models modality interactions in a shared Clifford space and is mathematically formulated via the spatial rotor $\mathcal{R}$ and geometric potential $\Phi$.
    Collectively, they modulate the principles of intra- and inter-modality
interaction within a modality-aware geometric manifold.
The rotor provides norm-preserving local transport, while the geometric
potential controls the contribution of transported signals, together
enabling stable modality alignment for MAGs.

\textbf{Modality-oriented Clifford Initialization.}
    To begin with, we need to initialize $\mathbf{X}$ to support CGP.
    Since modality attributes may have unequal dimensions $d_k$, a modality-specific encoder/projection $g_k:\mathbb{R}^{d_k}\!\to\mathbb{R}^{d}$ first maps them into a shared feature dimension, i.e., $\bar{x}_{u}^{(k)}=g_k(x_u^{(k)})$; $g_k$ is the identity when the dimensions already match.
    We then use stabilized per-modality normalization $\hat{x}_{u}^{(k)}=\bar{x}_{u}^{(k)}/\max(\|\bar{x}_{u}^{(k)}\|_2,\epsilon_x)$ and assign modality $k$ to the orthogonal Grade-1 basis $e_k\in\mathcal{C}l_K$:
    \begin{equation}
    \label{eq: modality-oriented Clifford initialization}
        \mathbf{H}_u^{(0)}\in\mathbb{R}^{d\times2^K}\coloneqq \frac{1}{\sqrt K}\sum_{k=1}^{K} \hat{x}_{u}^{(k)}e_k, \|\mathbf{H}_u^{(0)}\|_{\mathcal{C}l}\le 1.
    \end{equation}
    This definition also makes the interaction semantics explicit. For feature coordinate $j$, let $h_{u,j}=K^{-1/2}\sum_k\hat{x}_{u,j}^{(k)}e_k$. Its geometric product satisfies
    \begin{equation}
    \begin{aligned}
    \langle h_{u,j}h_{v,j}\rangle_0 &= \frac{1}{K}\sum_k \hat{x}_{u,j}^{(k)}\hat{x}_{v,j}^{(k)},\\
    \langle h_{u,j}h_{v,j}\rangle_2 &= \frac{1}{K}\sum_{p<q}(\hat{x}_{u,j}^{(p)}\hat{x}_{v,j}^{(q)}-\hat{x}_{u,j}^{(q)}\hat{x}_{v,j}^{(p)})e_pe_q.
    \end{aligned}
    \label{eq: component_mapping}
    \end{equation}
    Thus, Grade-0 aggregates same-modality compatibility, whereas each Grade-2 coefficient represents an oriented pairwise cross-modality discrepancy; an ordered modality pair is not identified one-to-one with a Clifford basis component. The formulation is algebraically defined for $K$ modalities, while all experiments in this paper use the prevalent $K=2$ text-image setting; empirical generality beyond two modalities is not claimed.

\textbf{Topology-oriented Geometric Potential.}
    After Clifford initialization, we capture the semantic curvature between connected nodes through feature-wise geometric interactions.
    For each feature coordinate $j\in\{1,\ldots,d\}$ and edge $(u,v)\in\mathcal E$, define
    \begin{equation}
    \label{eq: featurewise_interaction}
    \begin{aligned}
    \mathcal S_{uv,j}
    &=
    \left\langle
    h_{u,j}h_{v,j}
    \right\rangle_0,\\
    \mathcal B_{uv,j}
    &=
    \left\langle
    h_{u,j}h_{v,j}
    \right\rangle_2 .
    \end{aligned}
    \end{equation}
    Here $\mathcal S_{uv,j}$ measures same-modality compatibility, whereas $\mathcal B_{uv,j}$ characterizes the oriented pairwise cross-modality discrepancy defined in Eq.~(\ref{eq: component_mapping}).

    Based on these quantities, the modality-adaptive geometric potential for coordinate $j$ is
    \begin{equation}
    \label{eq: topology-oriented Clifford initialization}
        \Phi_{uv,j}
        =
        \exp\!\left(
        -
        \frac{
        \|\mathcal B_{uv,j}\|_{\mathcal Cl}^{2}
        }{
        |\mathcal S_{uv,j}|+\epsilon
        }
        \right),
        (u,v)\in\mathcal E .
    \end{equation}
    The scalar term measures semantic compatibility, while the bi-vector norm measures the magnitude of the corresponding cross-modality interaction plane.
    The exponential decay kernel therefore converts their relative discrepancy into a bounded feature-wise propagation weight.
    This construction provides the topology-aware interaction statistics used by the subsequent parallel transport.
    
\textbf{Training-free Geometric Propagation.}
    Based on the above feature-wise geometric statistics, we extend conventional graph propagation to Clifford-valued parallel transport.
    Specifically, $\mathcal S_{uv,j}$ measures local semantic compatibility, while $\mathcal B_{uv,j}$ determines the oriented interaction plane for geometric transport.
    For each edge $(u,v)$ and feature coordinate $j$, let
    \begin{equation}
    \label{eq: stabilized_rotor}
    \begin{aligned}
        \widehat{\mathcal B}_{uv,j}
        &=
        \frac{
        \mathcal B_{uv,j}
        }{
        \sqrt{
        \|\mathcal B_{uv,j}\|_{\mathcal Cl}^{2}
        +\epsilon_r^{2}
        }
        },\\
        \theta_{uv,j}
        &=
        \operatorname{atan2}
        \left(
        \|\mathcal B_{uv,j}\|_{\mathcal Cl},
        |\mathcal S_{uv,j}|+\epsilon_r
        \right),\\
        \mathcal R_{uv,j}
        &=
        \exp
        \left(
        -\frac{
        \theta_{uv,j}\widehat{\mathcal B}_{uv,j}
        }{2}
        \right).
    \end{aligned}
    \end{equation}
    Here, $\operatorname{atan2}(\cdot,\cdot)$ denotes the two-argument arctangent used to determine a stable rotation magnitude.
    The stabilized rotor continuously approaches the identity as
    $\|\mathcal B_{uv,j}\|_{\mathcal Cl}\rightarrow0$.

    The rotor determines how a neighboring representation is geometrically transported, while the geometric potential controls its contribution.
    We further introduce a self-loop potential
    $\bar{\Phi}_{uu,j}=1$ and set
    $\bar{\Phi}_{uv,j}=\Phi_{uv,j}$ for $(u,v)\in\mathcal E$.
    Feature-wise row normalization gives
    \begin{equation}
    \label{eq: potential_normalization}
        \widetilde{\Phi}_{uv,j}
        =
        \frac{
        \bar{\Phi}_{uv,j}
        }{
        \sum_{w\in\mathcal N(u)\cup\{u\}}
        \bar{\Phi}_{uw,j}
        }.
    \end{equation}

    Define the feature-wise transport
    \begin{equation}
    \label{eq: featurewise_transport}
        \mathcal T_{uv,j}(h)
        =
        \mathcal R_{uv,j}
        h
        \mathcal R_{uv,j}^{-1},
    \end{equation}
    with $\mathcal T_{uu,j}$ being the identity.
    Accordingly, the CGP update is
    \begin{equation}
    \label{eq: Training-free Geometric Propagation}
        h_{u,j}^{(l)}
        =
        \sum_{v\in\mathcal N(u)\cup\{u\}}
        \widetilde{\Phi}_{uv,j}
        \mathcal T_{uv,j}
        \left(
        h_{v,j}^{(l-1)}
        \right),
        \qquad j=1,\ldots,d.
    \end{equation}
    Equivalently, Eq.~(\ref{eq: Training-free Geometric Propagation}) defines a block-diagonal connection operator over the direct-sum feature space
    \begin{equation}
        \mathscr H
        =
        \bigoplus_{j=1}^{d}
        \mathcal Cl_K^{(1)}.
    \end{equation}
    Thus, each feature coordinate is transported by its own curvature-dependent Clifford connection, while all coordinates share the same graph topology.
    The normalized connection average keeps the propagation bounded, and the rotor sandwich preserves the Grade-1 state space, providing topology-aware modality components for the subsequent AHA module.

\textbf{Theoretical Analysis.}
    We first establish the geometric stability bound of the Clifford construction (Theorem~\ref{theorem1}; Supplementary Appendix A). We then analyze the normalized CGP operator in Eq.~(\ref{eq: Training-free Geometric Propagation}) under standard fixed-connection assumptions (Theorem~\ref{theorem2}; Supplementary Appendix B).
    \begin{theorem}
    \label{theorem1}
    \textbf{(Stability Bound of Clifford Manifold).}
    Let $f$ obtain $\mathbf H$ and $\mathcal A_{\mathcal G}=\{\Phi,\mathcal R\}$. Assume the modality encoders/projections are bounded Lipschitz maps and $\epsilon_x,\epsilon_r>0$ are fixed. For bounded perturbations of $(\mathbf X,\mathbf A)$,
    \begin{equation}
    \label{eq: theorem1}
    \begin{aligned}
    &\|\mathbf H-\mathbf H'\|_{\mathcal Cl}+\|\mathcal A_{\mathcal G}-\mathcal A'_{\mathcal G}\|_{\mathcal Cl}\\
    &\le K_{\rm map}(\|\mathbf X-\mathbf X'\|_{\mathcal Cl}+\gamma\|\mathbf A-\mathbf A'\|_{\mathcal Cl}).
    \end{aligned}
    \end{equation}
    where $K_{\rm map}$ and $\gamma$ depend on the encoder/projection bounds, graph degree, and stabilizers $\epsilon_x,\epsilon_r$.
    \end{theorem}
    \noindent\textit{Proof Sketch.}
    Stabilized normalization is Lipschitz for fixed $\epsilon_x>0$; the orthogonal Grade-1 lift is linear and norm bounded. The geometric product is bilinear, its grade projections are non-expansive, and the stabilized $\operatorname{atan2}$ rotor is continuous and Lipschitz on the resulting bounded domain, including the near-degenerate case. Combining these bounds with the edge perturbation term yields Eq.~(\ref{eq: theorem1}); details are in Supplementary Appendix A.

    \begin{theorem}
\label{theorem2}
\textbf{(Conditional Spectral Contraction of CGP).}
Assume an undirected graph, feature-wise symmetric potentials
$\Phi_{uv,j}=\Phi_{vu,j}$, reciprocal isometric transports
$\mathcal T_{vu,j}=\mathcal T_{uv,j}^{-1}$, and fixed edge
operators during one cached propagation stage.
Let $\mathcal P_{\mathcal G}$ denote the block-diagonal normalized
connection operator induced by
Eq.~(\ref{eq: Training-free Geometric Propagation})
on
$\mathscr H=\bigoplus_{j=1}^{d}\mathcal Cl_K^{(1)}$.
Let $\Pi$ be the orthogonal projection onto its unit-eigenspace,
and let $\rho<1$ denote the spectral radius of
$\mathcal P_{\mathcal G}$ restricted to the complementary invariant
subspace. Then
\begin{equation}
\label{eq: theorem2}
\left\|
\mathbf H^{(l)}-\Pi\mathbf H^{(0)}
\right\|_{D,\mathcal Cl}
\le
\rho^l
\left\|
\mathbf H^{(0)}-\Pi\mathbf H^{(0)}
\right\|_{D,\mathcal Cl}.
\end{equation}
\end{theorem}

\noindent\textit{Proof Sketch.}
For each feature coordinate, symmetry of the potential and
reciprocity of the rotor transport induce a self-adjoint normalized
connection averaging operator in the corresponding degree-weighted
inner product.
Their direct sum therefore defines a self-adjoint block-diagonal
operator $\mathcal P_{\mathcal G}$ on $\mathscr H$.
Decomposing $\mathbf H^{(0)}$ into the unit-eigenspace and its
orthogonal complement yields Eq.~(\ref{eq: theorem2}).
The statement is conditional on fixed reciprocal transports and on
$\rho<1$ on the complementary invariant subspace; it does not claim
an unconditional contraction rate for arbitrary time-varying
connections.
Supplementary Appendix B provides the complete derivation.

\subsection{Adaptive Holographic Aggregation}
\label{sec: Adaptive Holographic Aggregation}

\textbf{Motivation.}
    To unleash the representation potential of geometrically propagated features (i.e., aligned tokens), it is imperative to reveal their intricate dependencies.
    Unfortunately, prevalent message aggregation (e.g., simple concatenation or mean pooling) indiscriminately compresses modality-aware components and ignores receptive fields, leading to poor generalizability and performance degradation.
    We clarify that CGP maintains the propagated node states in the Grade-1 subspace because rotor sandwich transport preserves vector grade.
    The Grade-0 and Grade-2 quantities in Eq.~(\ref{eq: component_mapping}) are therefore edgewise interaction statistics that parameterize the geometric potential and spatial rotor, rather than additional propagated state channels.
    Consequently, the propagated Grade-1 components encode modality-specific states that have already been modulated by topology-aware intra- and inter-modality interactions during CGP.
    Based on this observation, we propose Adaptive Holographic Aggregation (AHA), which adaptively filters these propagated Clifford components according to their energy and subsequently reconciles their representations across propagation scales.

\textbf{Energy-aware Component Filtering.}
    Since $\mathbf H_u^{(0)}$ lies in the Grade-1 subspace and rotor sandwich transport preserves grade, the normalized CGP update in Eq.~(\ref{eq: Training-free Geometric Propagation}) satisfies
    \begin{equation}
    \label{eq: grade1_closure}
        \mathbf H_u^{(l)}
        =
        \sum_{k=1}^{K}\mathbf H_{u,k}^{(l)}e_k
        \in \mathcal{C}l_K^{(1)},
         l=0,\ldots,L.
    \end{equation}
    Here $\mathbf H_{u,k}^{(l)}\in\mathbb R^d$ denotes the propagated component associated with modality axis $e_k$.
    Although the propagated state remains Grade-1, each component has been transformed by the geometric potential and rotor constructed from the Grade-0 and Grade-2 interaction statistics in Eq.~(\ref{eq: component_mapping}).

    We quantify the information density of the active Clifford components by
    \begin{equation}
    \label{eq: energy_energy}
        \mathbb E_u^{(l)}
        =
        \left[
        \|\mathbf H_{u,1}^{(l)}\|_2^2,
        \ldots,
        \|\mathbf H_{u,K}^{(l)}\|_2^2
        \right]^{\top}
        \in\mathbb R^{K}.
    \end{equation}
    A learnable gate then adaptively modulates these propagated modality components:
    \begin{equation}
    \label{eq: energy_gating}
    \begin{aligned}
        \boldsymbol\alpha_u^{(l)}
        &=
        \sigma\!\left(
        \mathbf W_{\mathcal G}
        \operatorname{Norm}(\mathbb E_u^{(l)})
        +
        \mathbf b_{\mathcal G}
        \right),
        \widetilde{\mathbf H}_u^{(l)}
        =
        \sum_{k=1}^{K}
        \alpha_{u,k}^{(l)}
        \mathbf H_{u,k}^{(l)}e_k ,
    \end{aligned}
    \end{equation}
    where $\boldsymbol\alpha_u^{(l)}\in(0,1)^K$.
    The soft gate attenuates less informative propagated components rather than removing them deterministically.
    Importantly, AHA does not identify the $K$ propagated modality axes with the $K^2$ ordered modality interactions.
    Instead, cross-modality interactions affect these propagated states through the scalar and bi-vector statistics used to construct CGP, while AHA determines how the resulting topology-aware modality components contribute to downstream fusion.

\textbf{Scale-aware Resonance Fusion.}
    After obtaining the energy-filtered topology-aware modality components, their validity varies across topology scales corresponding to the propagation depth $L$ within each channel.
    To reconcile these receptive fields and facilitate fusion, we first generate $\mathbf{H}_u^{\text{ctx}}$ by learnable $\mathbf{W}_{\tau}$ to capture the consensus profile across all scales:
\begin{equation}
\label{eq: scale_weight}
\begin{aligned}
    \mathbf{H}_u^{\text{ctx}}\in\mathbb{R}^{d\times2^K}\coloneqq \text{Norm} \left(   \sum_{l=0}^{L} \left[\mathbf{W}_{\tau}^{(l)}\tilde{\mathbf{H}}_u^{(l)}+ \mathbf{b}_{\tau}^{(l)}\right] \right).
\end{aligned}
\end{equation}
    Subsequently, we employ a resonance attention mechanism to compute the scale validity score $\beta_{u,l}$, where a learnable attention score $\mathbf{a}_S^\top$ projects the interaction between the current representation and consensus profile into a scalar space.
    Based on this, AHA aggregates the grade-filtered and scale-weighted tokens, and the $d_f$-dimension fusion representation $\mathbf{Z}_u$ is obtained by projecting it back into Euclidean space via a Clifford linear layer parameterized by $\mathbf{W}_{\text{out}}$:
\begin{equation}
\label{eq: scale_aggregation}
\begin{aligned}
    &\;\;\;\mathbf{Z}_u\in\mathbb{R}^{d_f}\coloneqq \mathbf{W}_{\text{out}} \left( \sum_{l=0}^{L} \beta_{u,l} \cdot \tilde{\mathbf{H}}_u^{(l)} \right) + \mathbf{b}_{\text{out}},\\
    &\beta_{u,l} = \frac{\exp(\mathbf{a}_S^\top \cdot \tanh(\mathbf{W}_S [\mathbf{H}_u^{\text{ctx}} \| \tilde{\mathbf{H}}_u^{(l)}]))}{\sum_{k=0}^{L} \exp(\mathbf{a}_S^\top \cdot \tanh(\mathbf{W}_S [\mathbf{H}_u^{\text{ctx}} \| \tilde{\mathbf{H}}_u^{(k)}]))}.
\end{aligned}
\end{equation}

\textbf{Theoretical Analysis.}
    We characterize the sensitivity introduced by AHA without assuming that the consensus profile is an unknown downstream optimum. Let $f_{\mathrm{cll}}$ denote the affine Clifford projection in Eq.~(\ref{eq: scale_aggregation}), whose linear part has operator norm at most $\omega$, and let $\mathbf Z_u^{\mathrm{raw}}=f_{\mathrm{cll}}(\sum_l\beta_{u,l}\mathbf H_u^{(l)})$.
    \begin{theorem}
    \label{theorem3}
    \textbf{(AHA Stability and Consensus Bound).}
    Since $\beta_{u,l}\ge0$ and $\sum_l\beta_{u,l}=1$,
    \begin{equation}
    \label{eq: theorem3}
    \begin{aligned}
    \|\mathbf Z_u-\mathbf Z_u^{\mathrm{raw}}\|_2
    &\le\omega\sum_{l=0}^{L}\beta_{u,l}\|(\mathbf1-\boldsymbol\alpha_u^{(l)})\odot\mathbf H_u^{(l)}\|_{\mathcal Cl},\\
    \|\mathbf Z_u-f_{\mathrm{cll}}(\mathbf H_u^{\mathrm{ctx}})\|_2
    &\le\omega\sum_{l=0}^{L}\beta_{u,l}\|\tilde{\mathbf H}_u^{(l)}-\mathbf H_u^{\mathrm{ctx}}\|_{\mathcal Cl}.
    \end{aligned}
    \end{equation}
    \end{theorem}
    These inequalities separately bound the representation change introduced by grade gating and the deviation from the learned multiscale consensus. They do not assume that suppressed components are noise or that $\mathbf H_u^{\mathrm{ctx}}$ is the downstream-optimal representation. Both bounds follow from the linear part of the affine projection,
the convex scale weights, and the triangle inequality;
Supplementary Appendix C provides the complete proof.

\begin{algorithm}[t]
\caption{LION: c\underline{LI}ff\underline{O}rd \underline{N}eural Paradigm for Multimodal-Attributed Graphs}
\label{alg: lion}
\begin{algorithmic}[1]
\REQUIRE Multimodal-Attributed Graph $\mathcal{G}=(\mathcal{V}, \mathcal{E}, \{\mathbf{X}^{(m)}\}_{m\in\mathcal{M}})$, Model Layers $L$ (Graph Propagation Depth).
\ENSURE Learned Node- or Modality-level Representations $\mathbf{Z}$.

\STATE \textcolor{blue}{{/* Step 1: Clifford Geometric Propagation for Modality Alignment}}
\STATE \textcolor{blue}{// CGP obtains graph propagated features, which are equivalent to generating aligned tokens by modality alignment.}
\STATE Lift Euclidean modality-attributed features $\mathbf{X}$ into the Clifford manifold to obtain $\mathbf{H}^{(0)}$ via Eq.~(\ref{eq: modality-oriented Clifford initialization}).
\STATE Capture intricate semantic curvature to instantiate $\mathcal{A_G}\coloneqq\{\Phi,\mathcal{R}\}$ and then facilitate multimodal interactions via Eq.~(\ref{eq: topology-oriented Clifford initialization}). 
\STATE Perform modality-aware, curvature-adaptive high-order graph propagation to obtain $\{\mathbf{H}^{(1)},\dots,\mathbf{H}^{(L)}\}$ via Eq.~(\ref{eq: Training-free Geometric Propagation}).
\STATE \textcolor{blue}{{/* Step 2: Adaptive Holographic Aggregation for Modality Fusion}}
\STATE \textcolor{blue}{// AHA follows CGP to obtain node- or modality-level outputs for specific tasks, which is equivalent to modality fusion.}

\STATE Quantify the information density of modality interaction channels within the propagated features via Eq.~(\ref{eq: energy_energy}).
\STATE Dynamically modulate these channels based on the computed information density and a learnable gate via Eq.~(\ref{eq: energy_gating}).
\STATE \textcolor{blue}{// This module facilitates adaptive integration of modality interaction channels by energy-based information density.}
\STATE Capture the consensus profile across propagation depths via Eq.~(\ref{eq: scale_weight}). 
\STATE Dynamically modulate the contribution of each depth based on the consensus profile via Eq.~(\ref{eq: scale_aggregation}). 
\STATE \textcolor{blue}{// This module facilitates the adaptive integration of multi-scale receptive fields by semantic consensus.}
\STATE \textbf{Return} The task-adaptive fusion representation in Euclidean space.
\end{algorithmic}
\end{algorithm}

\noindent\textit{Proof Sketch.}
    The first inequality compares the gated and ungated convex scale aggregates and applies the operator-norm bound of $f_{\mathrm{cll}}$. The second inserts the learned consensus profile inside the same convex aggregate and applies the triangle inequality. No assumption is made about whether an attenuated component is signal or noise. The complete proof is provided in Supplementary Appendix C.

\begin{table*}[t]
\setlength{\abovecaptionskip}{-0.01cm}
\centering
\caption{The statistical information of the experimental datasets. The ``Tasks'' column summarizes applicable task types; evaluated dataset-task pairs are specified by the result-table headers.}
\label{table: datasets}
\resizebox{\linewidth}{!}{
\begin{tabular}{l|cccc|ll}
\midrule[0.3pt]
Datasets     & \# Modalities & \# Nodes & \# Edges   & \# Classes & Tasks                        & Description    \\ \midrule[0.3pt]
RedditS      & Text, Image  & 15,894  & 566,160   & 20              & Graph (Node, Link)           & Social Network \\
Movies       & Text, Image  & 16,672  & 218,390   & 20              & Graph (Node, Link)           & Movie Network  \\
Grocery      & Text, Image  & 17,074  & 171,340   & 20              & Graph (Node, Link), Modality & Recommendation \\
SemArt       & Text, Image  & 21,382  & 1,216,432 & -               & Graph (Link), Modality       & Art Network    \\
Flickr30k    & Text, Image  & 31,783  & 181,151   & -               & Graph (Link), Modality       & Image Network  \\
Sports       & Text, Image  & 50,250  & 356,202   & -               & Graph (Link), Modality       & Recommendation \\
Ele-fashion  & Text, Image  & 97,766  & 199,602   & 12              & Graph (Node, Link), Modality & Recommendation \\
Cloth        & Text, Image  & 125,839 & 951,271   & -               & Graph (Link), Modality       & Recommendation \\
Goodreads & Text, Image  & 685,294 & 7,235,048 & 11              & Graph (Node, Link)           & Book Network   \\ \midrule[0.3pt]
\end{tabular}}
\end{table*}

\subsection{Algorithm and Complexity Analysis}
\label{sec: Algorithm and Complexity Analysis}

\textbf{Algorithm Details}.
    For a more comprehensive presentation, we provide the complete LION in Algorithm~\ref{alg: lion}.
    Specifically, our framework transforms raw multimodal inputs into a unified representation space via two decoupled phases: Clifford Geometric Propagation (CGP) for modality alignment and Adaptive Holographic Aggregation (AHA) for modality fusion.
    CGP first lifts raw modalities into the Clifford manifold, constructs the geometric potential and spatial rotor from topology-aware interactions, and performs training-free high-order propagation to generate aligned multiscale tokens.
    AHA then filters topology-aware modality components by energy and fuses
propagation depths by resonance with a consensus profile.
    This design makes the expensive topology-aware geometric computation independent of iterative parameter learning, while retaining task-adaptive fusion in the trainable stage.
    More specifically, the Clifford lifting step assigns each modality to an orthogonal basis vector, which preserves modality identity while enabling explicit geometric products between modalities.
    The scalar and bi-vector components of the edgewise geometric products instantiate the intra- and inter-modality interaction statistics that parameterize CGP.
    After CGP, AHA operates on the resulting topology-aware Clifford token components and propagation depths to decide which components and receptive fields should dominate the final representation.
    In this sense, Algorithm~\ref{alg: lion} separates representation construction from task-specific optimization while keeping both stages connected through the same Clifford token space.
    The propagated tokens can be shared by node-, edge-, and modality-level heads, which makes the framework naturally adaptable to heterogeneous MAG tasks.
    Detailed algorithmic interpretation and complexity derivations are provided in Supplementary Appendix E.

\textbf{Theoretical Time-Space Complexity.} 
    Let $N=|\mathcal{V}|$, $M=|\mathcal{E}|$, $L$ be the propagation depth, and $D=2^K d$ be the Clifford multi-vector dimension.
    The one-time CGP pre-processing costs $\mathcal{O}(M\cdot 2^K\cdot D + L\cdot M\cdot D)$, while the trainable AHA phase costs $\mathcal{O}(L\cdot N\cdot D)$ per epoch.
    The cached multiscale representations require $\mathcal{O}(L\cdot N\cdot D)$ memory, and the learnable projection/gating parameters occupy $\mathcal{O}(D^2)$ space.
    Since CGP is parameter-free and cached offline, LION avoids recursive neighborhood expansion during training and remains asymptotically linear in graph size for the practical MAG settings considered here.
    This decoupled computation is particularly useful for large MAGs because propagation results can be reused across epochs and downstream heads.
    The trainable AHA module therefore operates on fixed multiscale Clifford tokens, which substantially reduces repeated graph traversal while preserving the ability to learn task-specific fusion weights.
    The memory overhead is dominated by cached propagated features, but this cost is predictable and can be controlled by the propagation depth $L$ and latent dimension $d$.
    This trade-off favors stable and efficient training when the same graph is evaluated across multiple downstream tasks.

\begin{table*}[t]
\setlength{\abovecaptionskip}{-0.01cm}
\centering
\caption{Performance comparison on graph downstream tasks.
The best/second result is \textbf{bold}/\underline{underline}.}
\label{tab: graph downstream performance comparison}
\footnotesize 
\renewcommand{\arraystretch}{1.2}
\resizebox{\linewidth}{!}{
\setlength{\tabcolsep}{1.5mm}{
\begin{tabular}{l!{\vrule width 0.1pt}cc!{\vrule width 0.1pt}cc!{\vrule width 0.1pt}cc!{\vrule width 0.1pt}cc!{\vrule width 0.1pt}cc!{\vrule width 0.1pt}cc}
\hline\thickhline
\rowcolor{gray!80}
\multicolumn{1}{c!{\vrule width 0.1pt}}{\textbf{\textcolor{white}{Tasks}}} & \multicolumn{4}{c!{\vrule width 0.1pt}}{\textbf{\textcolor{white}{Node Classification}}} & \multicolumn{4}{c!{\vrule width 0.1pt}}{\textbf{\textcolor{white}{Link Prediction}}} & \multicolumn{4}{c}{\textbf{\textcolor{white}{Node Clustering}}} \\
\hline
\rowcolor{gray!20}
 & \multicolumn{2}{c!{\vrule width 0.1pt}}{\textbf{Movies}} & \multicolumn{2}{c!{\vrule width 0.1pt}}{\textbf{Goodreads}} & \multicolumn{2}{c!{\vrule width 0.1pt}}{\textbf{Sports}} & \multicolumn{2}{c!{\vrule width 0.1pt}}{\textbf{Cloth}} & \multicolumn{2}{c!{\vrule width 0.1pt}}{\textbf{RedditS}} & \multicolumn{2}{c}{\textbf{Grocery}} \\
\cline{2-13}
\rowcolor{gray!20}
\multirow{-2}{*}{\diagbox[width=8.5em,height=2.4em]{\textbf{Methods}}{\textbf{Datasets}}} 
 & Acc & F1-Score & Acc & F1-Score & MRR & Hits@3 & MRR & Hits@3 & NMI & ARI & NMI & ARI \\
\hline
GCN & 39.12$_{ \pm 1.33}$ & 36.17$_{ \pm 1.41}$ & 56.03$_{ \pm 1.52}$ & 50.14$_{ \pm 1.21}$ & 46.23$_{ \pm 0.28}$ & 58.09$_{ \pm 0.73}$ & 40.52$_{ \pm 0.37}$ & 48.15$_{ \pm 0.44}$ & 62.18$_{ \pm 0.62}$ & 64.07$_{ \pm 0.51}$ & 38.13$_{ \pm 0.43}$ & 29.54$_{ \pm 0.38}$ \\
\rowcolor{gray!10} MMGCN & 44.87$_{ \pm 1.42}$ & 39.92$_{ \pm 1.37}$ & 58.54$_{ \pm 0.87}$ & 52.36$_{ \pm 1.53}$ & 48.44$_{ \pm 0.34}$ & 60.52$_{ \pm 0.58}$ & 42.83$_{ \pm 0.41}$ & 50.94$_{ \pm 0.56}$ & 72.10$_{ \pm 0.55}$ & 70.08$_{ \pm 0.62}$ & 45.84$_{ \pm 0.49}$ & 32.26$_{ \pm 0.41}$ \\
GAT & 40.25$_{ \pm 1.48}$ & 36.43$_{ \pm 1.52}$ & 57.12$_{ \pm 1.08}$ & 51.08$_{ \pm 1.24}$ & 47.19$_{ \pm 0.43}$ & 59.14$_{ \pm 0.77}$ & 41.26$_{ \pm 0.49}$ & 49.34$_{ \pm 0.58}$ & 65.45$_{ \pm 0.68}$ & 66.28$_{ \pm 0.63}$ & 39.42$_{ \pm 0.45}$ & 28.75$_{ \pm 0.47}$ \\
\rowcolor{gray!10} MGAT & 43.14$_{ \pm 1.37}$ & 38.09$_{ \pm 1.44}$ & 59.88$_{ \pm 0.91}$ & 53.75$_{ \pm 1.38}$ & 49.82$_{ \pm 0.36}$ & 62.03$_{ \pm 0.54}$ & 43.91$_{ \pm 0.45}$ & 52.54$_{ \pm 0.51}$ & 73.33$_{ \pm 0.53}$ & 69.14$_{ \pm 0.59}$ & 43.08$_{ \pm 0.43}$ & 31.15$_{ \pm 0.46}$ \\
\hline
MLaGA & 48.37$_{ \pm 1.02}$ & 42.08$_{ \pm 1.64}$ & 66.42$_{ \pm 1.25}$ & 60.29$_{ \pm 1.13}$ & 54.33$_{ \pm 0.62}$ & 68.41$_{ \pm 0.88}$ & 49.12$_{ \pm 0.73}$ & 58.37$_{ \pm 0.80}$ & 82.58$_{ \pm 1.15}$ & 82.14$_{ \pm 1.12}$ & 51.92$_{ \pm 0.83}$ & 37.58$_{ \pm 0.89}$ \\
\rowcolor{gray!10} GraphGPT-O & 49.12$_{ \pm 1.90}$ & 42.94$_{ \pm 1.13}$ & 64.25$_{ \pm 0.92}$ & 58.76$_{ \pm 1.30}$ & 55.82$_{ \pm 0.68}$ & 69.17$_{ \pm 0.75}$ & 51.43$_{ \pm 0.66}$ & 60.18$_{ \pm 0.73}$ & 79.33$_{ \pm 1.02}$ & 80.41$_{ \pm 0.96}$ & 50.64$_{ \pm 0.69}$ & 38.83$_{ \pm 0.62}$ \\
Graph4MM & 49.76$_{ \pm 1.85}$ & 43.22$_{ \pm 1.95}$ & 67.18$_{ \pm 0.78}$ & 61.35$_{ \pm 0.83}$ & 53.94$_{ \pm 0.71}$ & 67.58$_{ \pm 0.83}$ & 50.84$_{ \pm 0.68}$ & 59.43$_{ \pm 0.63}$ & 84.14$_{ \pm 0.92}$ & 83.25$_{ \pm 0.90}$ & 51.20$_{ \pm 0.78}$ & 38.92$_{ \pm 0.75}$ \\
\rowcolor{gray!10} InstructG2I & 50.57$_{ \pm 1.74}$ & 44.08$_{ \pm 1.38}$ & 65.84$_{ \pm 0.90}$ & 59.93$_{ \pm 1.02}$ & 56.75$_{ \pm 0.83}$ & 70.12$_{ \pm 0.93}$ & 52.28$_{ \pm 0.80}$ & 61.54$_{ \pm 0.86}$ & 82.83$_{ \pm 1.09}$ & 81.58$_{ \pm 1.05}$ & 49.87$_{ \pm 0.92}$ & 35.28$_{ \pm 0.82}$ \\
\hline
DMGC & 51.73$_{ \pm 1.95}$ & 46.84$_{ \pm 1.59}$ & 63.42$_{ \pm 1.27}$ & 61.18$_{ \pm 1.63}$ & 56.47$_{ \pm 0.54}$ & 69.83$_{ \pm 0.49}$ & 51.87$_{ \pm 0.61}$ & 60.84$_{ \pm 0.57}$ & \underline{89.62}$_{ \pm 0.73}$ & \underline{89.14}$_{ \pm 0.80}$ & \underline{56.13}$_{ \pm 0.66}$ & 41.82$_{ \pm 0.60}$ \\
\rowcolor{gray!10} DGF & 52.94$_{ \pm 1.64}$ & 45.13$_{ \pm 1.70}$ & 66.83$_{ \pm 1.61}$ & 62.54$_{ \pm 1.12}$ & 55.93$_{ \pm 0.51}$ & 69.17$_{ \pm 0.54}$ & 51.24$_{ \pm 0.63}$ & 60.13$_{ \pm 0.59}$ & 88.87$_{ \pm 0.70}$ & 88.42$_{ \pm 0.76}$ & {55.88}$_{ \pm 0.63}$ & \underline{42.25}$_{ \pm 0.68}$ \\
MIG-GT & 53.17$_{ \pm 2.02}$ & \underline{47.60}$_{ \pm 1.89}$ & {68.24}$_{ \pm 0.90}$ & {64.87}$_{ \pm 1.24}$ & 59.13$_{ \pm 0.47}$ & 71.94$_{ \pm 0.44}$ & {55.12}$_{ \pm 0.56}$ & \underline{63.24}$_{ \pm 0.53}$ & 86.25$_{ \pm 0.82}$ & 84.89$_{ \pm 0.88}$ & 51.62$_{ \pm 0.70}$ & 38.18$_{ \pm 0.66}$ \\
\rowcolor{gray!10} NTSFormer & \underline{53.89}$_{ \pm 2.16}$ & {46.94}$_{ \pm 1.93}$ & 71.19$_{ \pm 0.93}$ & \underline{65.80}$_{ \pm 0.80}$ & \underline{60.52}$_{ \pm 0.44}$ & \underline{72.28}$_{ \pm 0.41}$ & 54.83$_{ \pm 0.53}$ & 62.94$_{ \pm 0.50}$ & 85.81$_{ \pm 0.78}$ & 84.14$_{ \pm 0.83}$ & 52.33$_{ \pm 0.68}$ & 37.81$_{ \pm 0.63}$ \\
UniGraph2 & 52.92$_{ \pm 1.79}$ & 46.28$_{ \pm 1.66}$ & \underline{72.82}$_{ \pm 0.82}$ & {65.32}$_{ \pm 0.96}$ & 59.84$_{ \pm 0.49}$ & 71.26$_{ \pm 0.47}$ & \underline{55.31}$_{ \pm 0.54}$ & 63.18$_{ \pm 0.52}$ & 86.56$_{ \pm 0.68}$ & 85.24$_{ \pm 0.73}$ & 54.42$_{ \pm 0.64}$ & 40.86$_{ \pm 0.62}$ \\
\rowcolor{gray!10} {LION (Ours)} & \textbf{58.61}$_{ \pm 1.28}$ & \textbf{51.73}$_{ \pm 1.75}$ & \textbf{78.54}$_{ \pm 0.70}$ & \textbf{68.91}$_{ \pm 1.13}$ & \textbf{62.31}$_{ \pm 0.55}$ & \textbf{73.87}$_{ \pm 0.39}$ & \textbf{58.47}$_{ \pm 0.85}$ & \textbf{66.58}$_{ \pm 0.48}$ & \textbf{90.53}$_{ \pm 0.89}$ & \textbf{90.17}$_{ \pm 1.08}$ & \textbf{58.54}$_{ \pm 0.72}$ & \textbf{46.12}$_{ \pm 0.89}$ \\
\hline\thickhline
\end{tabular}
}}
\end{table*}

\begin{table}[t]
\setlength{\abovecaptionskip}{-0.01cm}
\centering
\caption{Performance comparison on modality-level tasks.}
\label{tab: multimodal downstream performance comparison}
\footnotesize 
\renewcommand{\arraystretch}{1.2}
\resizebox{\linewidth}{!}{
\setlength{\tabcolsep}{1.2mm}{
\begin{tabular}{l!{\vrule width 0.1pt}cc!{\vrule width 0.1pt}cc!{\vrule width 0.1pt}cc}
\hline\thickhline
\rowcolor{gray!80}
\multicolumn{1}{c!{\vrule width 0.1pt}}{\textbf{\textcolor{white}{Tasks}}} & \multicolumn{2}{c!{\vrule width 0.1pt}}{\textbf{\textcolor{white}{Modal Retrieval}}} & \multicolumn{2}{c!{\vrule width 0.1pt}}{\textbf{\textcolor{white}{G2Text}}} & \multicolumn{2}{c}{\textbf{\textcolor{white}{G2Image}}} \\
\hline

\rowcolor{gray!20}
 & \multicolumn{2}{c!{\vrule width 0.1pt}}{\textbf{Ele-fashion}} & \multicolumn{2}{c!{\vrule width 0.1pt}}{\textbf{Flickr30k}} & \multicolumn{2}{c}{\textbf{SemArt}} \\
\cline{2-7}

\rowcolor{gray!20}
\multirow{-2}{*}{\diagbox[width=8em,height=2.4em]{\textbf{Methods}}{\textbf{Datasets}}} 
 & \;MRR\; & Hits@3
 & BLEU-4 & \;CIDEr\;
 & \;CLIP-S\; & DINOv2-S \\
\hline

GCNII & 77.24 & 70.15 & 5.42 & 39.54 & 50.52 & 35.81 \\
\rowcolor{gray!10} MMGCN & 81.85 & 74.58 & 6.15 & 44.82 & 54.88 & 39.87 \\
GATv2 & 78.53 & 71.32 & 5.56 & 40.24 & 51.25 & 36.54 \\
\rowcolor{gray!10} MGAT & 82.41 & 75.26 & 6.28 & 45.56 & 55.43 & 40.52 \\
\hline
MLaGA & 87.65 & 79.85 & 9.54 & 71.58 & 68.52 & 53.15 \\
\rowcolor{gray!10}  GraphGPT-O & 88.45 & 80.50 & 9.89 & 72.58 & \underline{70.84} & 54.29 \\
Graph4MM & 86.25 & 78.42 & \underline{10.42} & \underline{74.82} & 67.21 & 52.56 \\
\rowcolor{gray!10}  InstructG2I & 89.12 & 81.24 & 9.68 & 71.98 & 69.92 & \underline{55.43} \\
\hline
DMGC & 91.20 & 83.12 & 7.83 & 61.82 & 61.54 & 47.26 \\
\rowcolor{gray!10}  DGF & 91.55 & 83.47 & 7.92 & 62.20 & 61.98 & 47.68 \\
MIG-GT & 92.54 & 84.51 & 8.15 & 63.57 & 63.82 & 48.94 \\
\rowcolor{gray!10}  NTSFormer & 92.88 & 84.93 & 8.24 & 64.12 & 63.26 & 48.56 \\
Unigraph2 & \underline{93.13} & \underline{85.45} & 8.56 & 65.28 & 64.56 & 49.83 \\
\rowcolor{gray!10} {LION (Ours)} & \textbf{94.67} & \textbf{86.92} & \textbf{11.54} & \textbf{78.92} & \textbf{74.21} & \textbf{58.30} \\
\hline\thickhline
\end{tabular}
}}
\end{table}

\section{Experiments}
\label{sec: experiments}
    In this section, we first introduce the experimental setup, and additional details for reproducibility are provided in the supplementary material.
    Then, we provide a comprehensive evaluation to address the following questions:
    \textbf{Q1 (Effectiveness)}. 
    How does LION perform as a new neural paradigm for MAG?
    \textbf{Q2 (Ablation)}. 
    If LION is effective, what contributes to its performance? 
    \textbf{Q3 (Interpretability)}. 
    In depth, how do these components exert their influence?
    \textbf{Q4 (Robustness)}. 
    How robust is LION when dealing with sparse scenarios?
    \textbf{Q5 (Scalability)}. 
    What are the time overhead, space overhead, and convergence efficiency of LION?
    
\subsection{Experimental Setup}

\textbf{Datasets}.
    In our experiments, we evaluate LION across 9 publicly available MAG datasets from 6 domains, achieving a comprehensive validation. 
    Specifically, these datasets include social network (RedditS)~\cite{desai1_RedditS}, movie network (Movies)~\cite{ni2019_Grocery_Cloth_Ele_Movies_Sports}, 4 recommendation networks (Grocery, Sports, Ele-fashion, Cloth)~\cite{ni2019_Grocery_Cloth_Ele_Movies_Sports,hou2024_Cloth_Ele_Sports}, art network (SemArt)~\cite{garcia2018_SemArt}, image network (Flickr30k)~\cite{plummer2015_Flickr30k}, and book network (Goodreads)~\cite{wan2018_Goodreads_NC,wan2019_Goodreads_NC}. 
    The dataset statistics are summarized in Table~\ref{table: datasets}.
    These datasets cover social, movie, recommendation, art, image, and book networks, supporting both graph-centric and modality-oriented tasks.
    Their graph scales range from medium-sized image-text and product graphs to the large Goodreads network with more than 0.68M nodes and 7.23M edges, which allows us to evaluate both representation quality and scalability.
    Their task coverage also spans supervised, unsupervised, retrieval, and generation settings, preventing the evaluation from being biased toward a single form of multimodal supervision.
    Detailed descriptions of each dataset are provided in Supplementary Appendix F.

\textbf{Baselines}.
    To achieve a comprehensive comparison, we utilize 
    (i) Single-modality GNN: GCN, GAT, GCNII, GATv2;
    (ii) Simple MAGNN: MMGCN~\cite{hu2021mmgcn}, MGAT~\cite{tao2020mgat};
    (iii) Conventional MAGNN: MLaGA~\cite{fan2025mlaga}, GraphGPT-O~\cite{fang2025graphgpt_o}, Graph4MM~\cite{ning2025graph4mm}, InstructG2I~\cite{jin2024instructg2i};
    (iv) Graph-enhanced MAGNN: DMGC~\cite{guo2025DMGC}, DGF~\cite{zheng2025dgf}, MIG-GT~\cite{hu2025mig_gt}, NTSFormer~\cite{hu2025ntsformer}, UniGraph2~\cite{he2025unigraph2}. 
    For fair comparison, all baselines are evaluated under the same dataset splits and task protocols whenever their original implementations support the corresponding downstream task.
    Details are shown in Supplementary Appendix G.

\textbf{Evaluation Protocols}.
    Fundamentally, the research motivation for MAG is to enhance data quality for traditional graph tasks and expand the graph downstream task spectrum to improve practical utility. 
    Therefore, evaluation involves
    (i) Prevalent graph tasks: node classification, link prediction, node clustering;
    (ii) Novel graph-oriented modality tasks: modality retrieval, Graph-to-Text (G2Text), Graph-to-Image (G2Image). 
    Given the complexity of the evaluation pipelines, hyperparameter settings, and quantitative metrics, please refer to Supplementary Appendix H for comprehensive details. All nine benchmarks provide text and image attributes. Graph-task and ablation entries shown with ``$\pm$'' are mean$\pm$std over five independent runs; modality retrieval/generation entries without ``$\pm$'' follow their benchmark single-run protocols.

\textbf{Experiment Environment}.
    Experiments are conducted on a workstation equipped with Intel Xeon Scalable processors and NVIDIA RTX 6000 Ada Generation GPUs with 96GB of VRAM, supported by 256GB of system RAM.
    The computational environment utilizes CUDA 12.9, while software implementations are developed using Python 3.10.18 and PyTorch 2.8.

\subsection{Performance Comparison}
\label{sec: Performance Comparison}

\textbf{Graph Tasks}.
    To answer \textbf{Q1}, we first report performance comparison on graph tasks.
    As shown in Table~\ref{tab: graph downstream performance comparison}, LION consistently achieves superior performance across all datasets, tasks, and metrics.
    Specifically, in node classification, LION outperforms the leading NTSFormer and Unigraph2 by a significant average margin of 5.84\%.
    This success highlights the capacity of LION to capture topology and modality insights through the Clifford geometric manifold.
    The gains are consistent across both classification and link-level prediction, indicating that LION does not merely improve a single supervised objective but learns transferable graph representations.
    This is important for MAGs because downstream labels are often sparse, while relational and multimodal cues are available at different granularities.
    By aligning modality attributes through topology-aware geometric propagation, LION can exploit these cues before the task-specific training stage.
    
    Regarding the baselines, graph-enhanced methods like NTSFormer generally outperform conventional methods such as MLaGA, yet they remain inferior to LION primarily due to their reliance on suboptimal Euclidean spaces.
    Furthermore, most approaches fail to maintain consistent dominance across all three tasks and six metrics. 
    For instance, while clustering-specific DMGC and DGF exhibit competitiveness performance, they fail to generalize effectively to the node classification task.
    Meanwhile, there are currently no specialized models designed for the link-level task.
    These observations suggest that simply attaching graph modules to multimodal encoders is insufficient for a general MAG neural paradigm.
    A unified representation space is needed so that the same model can support node-, edge-, and cluster-level reasoning without redesigning task-specific interaction modules.

\textbf{Modality Tasks}.
    Table~\ref{tab: multimodal downstream performance comparison} further validates the superiority of LION in modality tasks, encompassing prevalent retrieval and generation, where it achieves SOTA performance across all evaluation metrics.
    A notable performance gain is observed on the SemArt dataset for the G2Image task, where LION outperforms the leading baselines, GraphGPT-O and InstructG2I, by 4.76\% in CLIP-S and 5.18\% in DINOv2-S. 
    These results demonstrate that our alignment-then-fusion MAG neural paradigm effectively preserves the semantic integrity of multimodal attributes.
    This strategy facilitates more coherent and accurate content generation compared to existing methods that treat modality and topology as separate spaces.
    Notably, modality-agnostic prevalent GNNs and early MMGNNs fail to achieve competitive performance due to their outdated model architecture designs.
    The retrieval and generation results further show that topology is useful not only for graph prediction but also for modality-level semantic grounding.
    In retrieval tasks, aligned Clifford tokens improve cross-modal matching by reducing the discrepancy between visual and textual neighborhoods.
    In generation tasks, topology-aware fusion supplies structured context that guides the decoder or diffusion backbone toward outputs that are semantically compatible with the graph.
    Therefore, LION turns graph structure into a reusable multimodal prior rather than treating it as an additional feature appended to the input.

\begin{table}[t]
\setlength{\abovecaptionskip}{-0.01cm}
\centering
\caption{Ablation Study.}
\label{tab: ablation_study}
\renewcommand{\arraystretch}{1.2}
\resizebox{\linewidth}{!}{
\setlength{\tabcolsep}{0.8mm}{
\begin{tabular}{l!{\vrule width 0.1pt}cc!{\vrule width 0.1pt}cc!{\vrule width 0.1pt}cc!{\vrule width 0.1pt}cc}
\hline\thickhline
\rowcolor{gray!80}
\multicolumn{1}{c!{\vrule width 0.1pt}}{\textbf{\textcolor{white}{Tasks}}} & \multicolumn{4}{c!{\vrule width 0.1pt}}{\textbf{\textcolor{white}{Node Classification}}} & \multicolumn{4}{c}{\textbf{\textcolor{white}{Link Prediction}}} \\
\hline
\rowcolor{gray!20}
 & \multicolumn{2}{c!{\vrule width 0.1pt}}{\textbf{RedditS}} & \multicolumn{2}{c!{\vrule width 0.1pt}}{\textbf{Grocery}} & \multicolumn{2}{c!{\vrule width 0.1pt}}{\textbf{Flickr30k}} & \multicolumn{2}{c}{\textbf{SemArt}} \\
\cline{2-9}
\rowcolor{gray!20}
\multirow{-2}{*}{\diagbox[width=8em,height=2.4em]{\textbf{Methods}}{\textbf{Datasets}}} 
 & Acc & F1 & Acc & F1 & MRR & Hits@3 & MRR & Hits@3 \\
\hline
LION (Ours) & \textbf{94.8$_{ \pm 0.4}$} & \textbf{90.6$_{ \pm 0.5}$} & \textbf{88.3$_{ \pm 0.7}$} & \textbf{78.1$_{ \pm 0.9}$} & \textbf{68.4$_{ \pm 0.5}$} & \textbf{76.5$_{ \pm 0.4}$} & \textbf{85.2$_{ \pm 0.3}$} & \textbf{89.6$_{ \pm 0.2}$} \\
\rowcolor{gray!10} w/o Rotor & 93.9$_{ \pm 0.3}$ & 89.4$_{ \pm 0.4}$ & 87.4$_{ \pm 0.6}$ & 77.6$_{ \pm 0.8}$ & 67.8$_{ \pm 0.5}$ & 76.0$_{ \pm 0.4}$ & 84.7$_{ \pm 0.3}$ & 89.1$_{ \pm 0.2}$ \\
w/o Potential & 93.4$_{ \pm 0.4}$ & 88.7$_{ \pm 0.5}$ & 86.6$_{ \pm 0.7}$ & 76.8$_{ \pm 0.8}$ & 67.2$_{ \pm 0.4}$ & 75.7$_{ \pm 0.5}$ & 84.3$_{ \pm 0.3}$ & 88.8$_{ \pm 0.2}$ \\
\rowcolor{gray!10} w/o Energy & 92.8$_{ \pm 0.5}$ & 88.3$_{ \pm 0.6}$ & 86.2$_{ \pm 0.8}$ & 75.9$_{ \pm 1.0}$ & 66.5$_{ \pm 0.6}$ & 74.5$_{ \pm 0.5}$ & 83.5$_{ \pm 0.4}$ & 87.5$_{ \pm 0.3}$ \\
w/o Consen. & 92.6$_{ \pm 0.4}$ & 88.5$_{ \pm 0.6}$ & 85.9$_{ \pm 0.7}$ & 76.6$_{ \pm 0.9}$ & 67.0$_{ \pm 0.5}$ & 74.8$_{ \pm 0.5}$ & 82.9$_{ \pm 0.4}$ & 87.9$_{ \pm 0.3}$ \\
\rowcolor{gray!10} w/o Scale & 91.7$_{ \pm 0.6}$ & 86.8$_{ \pm 0.7}$ & 84.5$_{ \pm 0.9}$ & 75.5$_{ \pm 1.1}$ & 66.8$_{ \pm 0.7}$ & 73.5$_{ \pm 0.6}$ & 82.5$_{ \pm 0.5}$ & 86.7$_{ \pm 0.4}$ \\

\hline\thickhline
\rowcolor{gray!80}
\multicolumn{1}{c!{\vrule width 0.1pt}}{\textbf{\textcolor{white}{Tasks}}} & \multicolumn{4}{c!{\vrule width 0.1pt}}{\textbf{\textcolor{white}{Node Clustering}}} & \multicolumn{4}{c}{\textbf{\textcolor{white}{Multimodal Retrieval}}} \\
\hline
\rowcolor{gray!20}
 & \multicolumn{2}{c!{\vrule width 0.1pt}}{\textbf{Movies}} & \multicolumn{2}{c!{\vrule width 0.1pt}}{\textbf{Ele-Fashion}} & \multicolumn{2}{c!{\vrule width 0.1pt}}{\textbf{Sports}} & \multicolumn{2}{c}{\textbf{Cloth}} \\
\cline{2-9}
\rowcolor{gray!20}
\multirow{-2}{*}{\diagbox[width=8em,height=2.4em]{\textbf{Methods}}{\textbf{Datasets}}} 
 & NMI & ARI & NMI & ARI & MRR & Hits@3 & MRR & Hits@3 \\
\hline
LION (Ours) & \textbf{24.6$_{ \pm 0.4}$} & \textbf{10.6$_{ \pm 0.4}$} & \textbf{56.8$_{ \pm 0.6}$} & \textbf{48.9$_{ \pm 0.8}$} & \textbf{95.2$_{ \pm 0.2}$} & \textbf{88.7$_{ \pm 0.1}$} & \textbf{92.5$_{ \pm 0.1}$} & \textbf{89.1$_{ \pm 0.2}$} \\
\rowcolor{gray!10} w/o Rotor & 23.7$_{ \pm 0.3}$ & 10.0$_{ \pm 0.3}$ & 56.0$_{ \pm 0.5}$ & 47.9$_{ \pm 0.7}$ & 94.3$_{ \pm 0.2}$ & 88.1$_{ \pm 0.1}$ & 91.8$_{ \pm 0.1}$ & 88.6$_{ \pm 0.2}$ \\
w/o Potential & 23.5$_{ \pm 0.4}$ & 9.4$_{ \pm 0.3}$ & 55.9$_{ \pm 0.5}$ & 48.1$_{ \pm 0.7}$ & 93.9$_{ \pm 0.2}$ & 87.9$_{ \pm 0.1}$ & 90.5$_{ \pm 0.1}$ & 87.8$_{ \pm 0.2}$ \\
\rowcolor{gray!10} w/o Energy & 22.8$_{ \pm 0.5}$ & 8.9$_{ \pm 0.4}$ & 54.6$_{ \pm 0.7}$ & 46.9$_{ \pm 0.9}$ & 93.8$_{ \pm 0.3}$ & 87.0$_{ \pm 0.2}$ & 90.9$_{ \pm 0.2}$ & 87.1$_{ \pm 0.3}$ \\
w/o Consen. & 23.2$_{ \pm 0.4}$ & 8.7$_{ \pm 0.3}$ & 54.8$_{ \pm 0.6}$ & 46.5$_{ \pm 0.8}$ & 93.5$_{ \pm 0.3}$ & 86.4$_{ \pm 0.2}$ & 90.4$_{ \pm 0.2}$ & 87.3$_{ \pm 0.3}$ \\
\rowcolor{gray!10} w/o Scale & 22.9$_{ \pm 0.5}$ & 8.2$_{ \pm 0.2}$ & 54.2$_{ \pm 0.8}$ & 45.9$_{ \pm 1.0}$ & 93.0$_{ \pm 0.4}$ & 86.3$_{ \pm 0.3}$ & 90.1$_{ \pm 0.3}$ & 86.8$_{ \pm 0.4}$ \\

\hline\thickhline
\end{tabular}
}}
\end{table}

\subsection{Ablation Study}
\label{sec: Ablation Study}
    To answer \textbf{Q2}, we conduct an ablation study to investigate the contribution of each module within LION.  
    As detailed in Sec.~\ref{sec: Methodology}, LION consists of two core modules, namely CGP and AHA. 
    To evaluate their impact, we define five module variants. 
    For CGP in Eq.~(\ref{eq: Training-free Geometric Propagation}), (1) w/o Rotor and (2) w/o Potential remove the spatial rotor $\mathcal{R}$ and geometric potential $\Phi$ . 
    For AHA in Eq.~(\ref{eq: energy_gating}), Eq.~(\ref{eq: scale_weight}), and Eq.~(\ref{eq: scale_aggregation}), (3) w/o Energy, (4) w/o Consen., and (5) w/o Scale omit the energy-aware gating, consensus profile, and scale-aware resonance fusion.

    Based on this, Table~\ref{tab: ablation_study} confirms the necessity of module design. 
    Regarding the CGP, removing the spatial rotor and geometric potential leads to a significant performance drop. 
    This verifies that modeling topology as geometric rotations and differentiating interaction principles via semantic curvature is crucial for effective alignment. 
    The two CGP variants also reveal complementary roles of the rotor and potential.
    The rotor determines how neighbor semantics are transported into the target tangent space, while the potential controls the strength of curvature-aware interaction.
    Removing either component weakens the geometric interpretation of propagation and reduces the quality of aligned tokens.
    Regarding the AHA, the w/o Energy variant exhibits poor performance.
This demonstrates that indiscriminately aggregating all propagated
modality components introduces redundancy and confirms that
component-wise filtering based on information density is critical.
    Furthermore, the performance degradation in w/o Scale and w/o Consen. confirms that adaptive fusion guided by the consensus profile allows the model to reconcile receptive fields, which is beneficial for different propagated representations.
    These results support the view that modality fusion should consider both channel reliability and topology scale, rather than relying on a fixed pooling operation over all propagated features.

\begin{figure*}[t]
  \includegraphics[width=\textwidth]{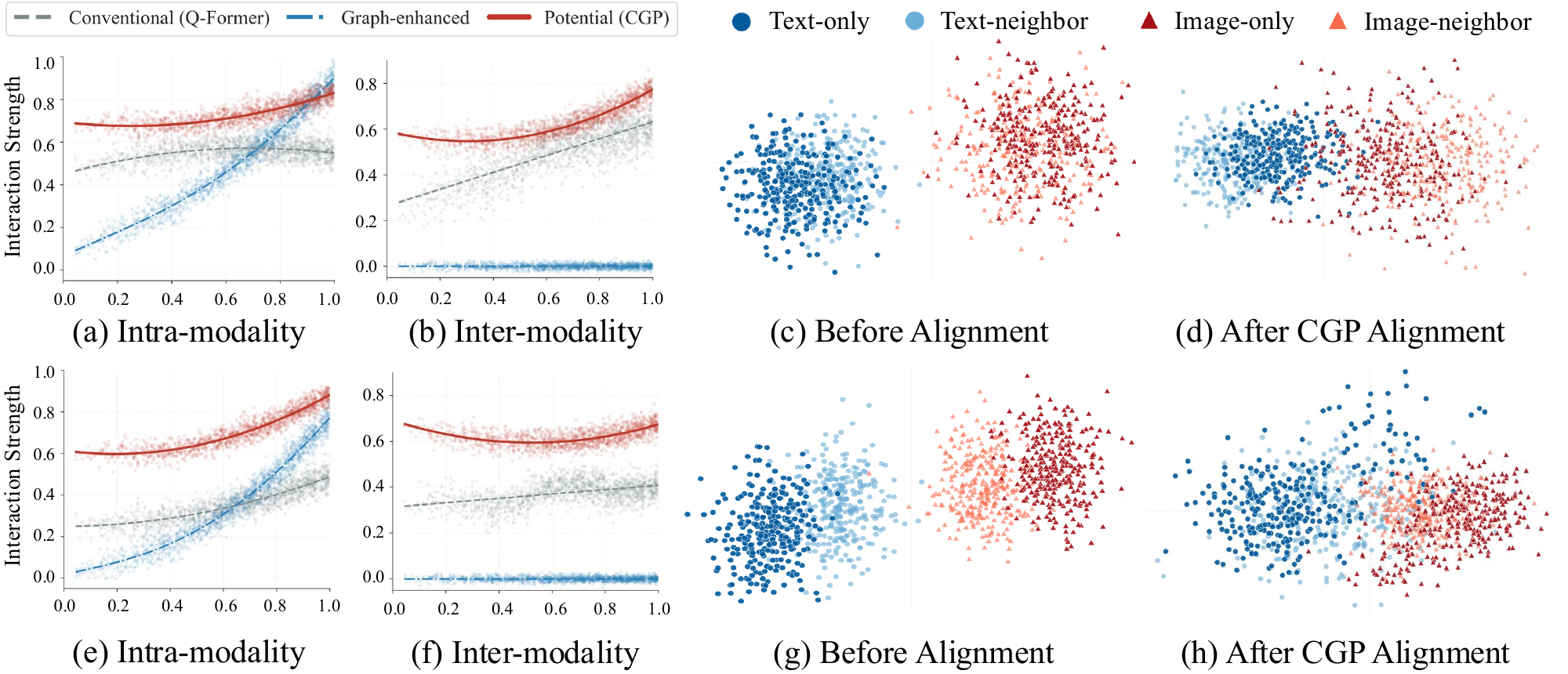}
  \caption{(a)-(d) The G2Text visualization of CGP modality interaction and alignment (i.e., geometric-induced adaptive graph propagation) on Flickr30k.
  (e)-(h) The G2Image visualization of CGP modality interaction and alignment on SemArt. 
  }
  \label{fig: cgp_alignment}
\end{figure*}

\begin{figure}[t]
  \centering
  \includegraphics[width=\linewidth]{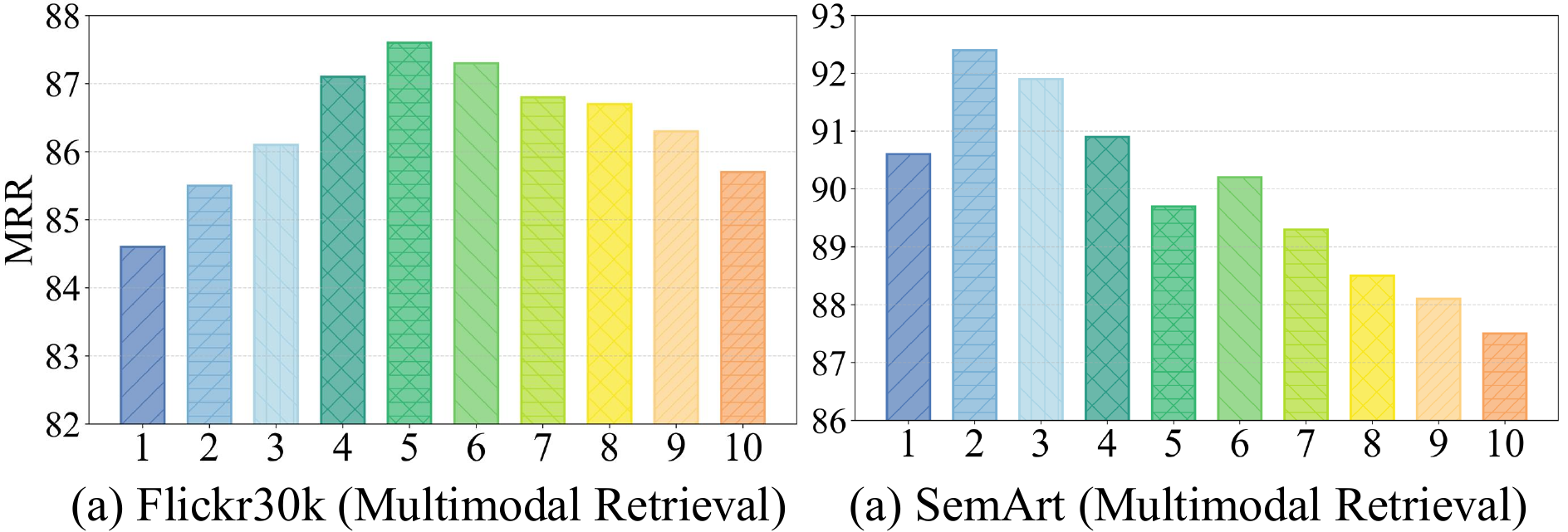}
  \caption{Hyperparameter analysis of CGP depth.}
  \label{fig: cgp_depth}
\end{figure}

\subsection{In-depth Analysis}
\label{sec: In-depth Analysis}

\textbf{CGP Module}.
    To answer \textbf{Q3}, we provide a visual analysis to intuitively interpret the underlying mechanisms of the CGP and AHA modules.
    First of all, we plot the intra- and inter-modality interaction strength (y-axis) against connected node feature similarity (x-axis). For CGP, the plotted quantities are $I_{\rm intra}=\|\langle h_uh_v\rangle_0\|/(\|h_u\|\|h_v\|+\epsilon)$ and $I_{\rm inter}=\|\langle h_uh_v\rangle_2\|/(\|h_u\|\|h_v\|+\epsilon)$; for baselines, we use the corresponding normalized message/attention coefficient, with zero only when no cross-modal message is defined. 
    As evident in Fig.~\ref{fig: cgp_alignment} (a)-(b) for Flickr30k and Fig.~\ref{fig: cgp_alignment} (e)-(f) for SemArt, the modality interaction baselines exhibit a sharp decay in interaction strength as node feature similarity decreases, which reveals a heavy reliance on the homophily assumption.
    Notably, the modality interactions in graph-enhanced methods occur exclusively within modality-specific views, resulting in an inter-modality interaction strength of zero.

    In contrast, CGP breaks the node homophily assumption by modality-aware geometric manifolds and topology-adaptive high-order graph propagation. 
    Specifically, by harmonizing modality interaction principles (Sec.~\ref{sec: Clifford Geometric Propagation}), CGP effectively facilitates both intra- and inter-modality interactions.
    Based on this, we visualize the token latent space before and after alignment as an interpretive diagnostic of how curvature-aware multimodality interaction reshapes the representation; the task results and ablations provide the quantitative performance evidence. 
    The visualization also helps explain why the proposed propagation differs from conventional homophily smoothing.
    Instead of forcing all neighboring nodes to become similar in a scalar space, CGP transports modality-aware geometric bases according to the learned curvature.
    This allows semantically compatible signals to be aligned while preserving the distinction between different interaction channels.
    
    Prior to alignment, the latent spaces of modalities appear disjointed. 
    After CGP, a distinct trend emerges where entity representations transcend modality boundaries to capture high-level semantic dependencies. 
    This strategy effectively gathers semantically related modalities while distancing others.
    Meanwhile, G2Text and G2Image tasks exhibit enhanced text and image compactness, reflecting an adaptive alignment that conforms to the specific downstream requirements.
    This result intuitively demonstrates that our spatial rotor and geometric potential successfully guide the parallel transport of modality attribute vectors for efficient alignment.
    Furthermore, Fig.~\ref{fig: cgp_depth} illustrates the impact of varying CGP propagation depths on model performance. 
    Experimental results suggest that a smaller propagation depth is preferable for dense graphs, while sparse graphs benefit from increased depth to capture a sufficient receptive field.

\begin{figure}[t]
  \centering
  \includegraphics[width=\linewidth]{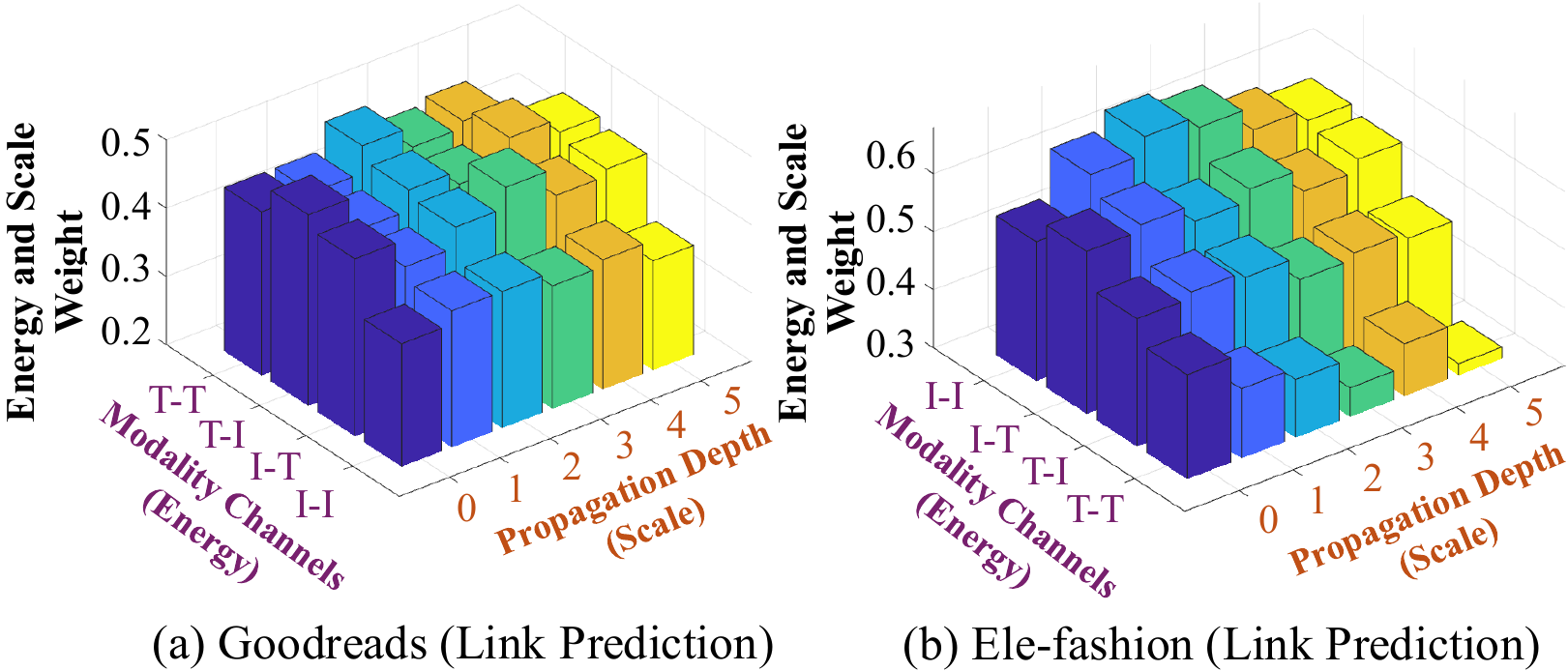}
  \caption{AHA visualization.
  T and I are text and image; T-T, T-I, I-T, and I-I denote ordered modality-pair diagnostics rather than one-to-one Clifford basis blades.}
  \label{fig: aha_fusion}
\end{figure}

\begin{figure*}[t]
  \includegraphics[width=\textwidth]{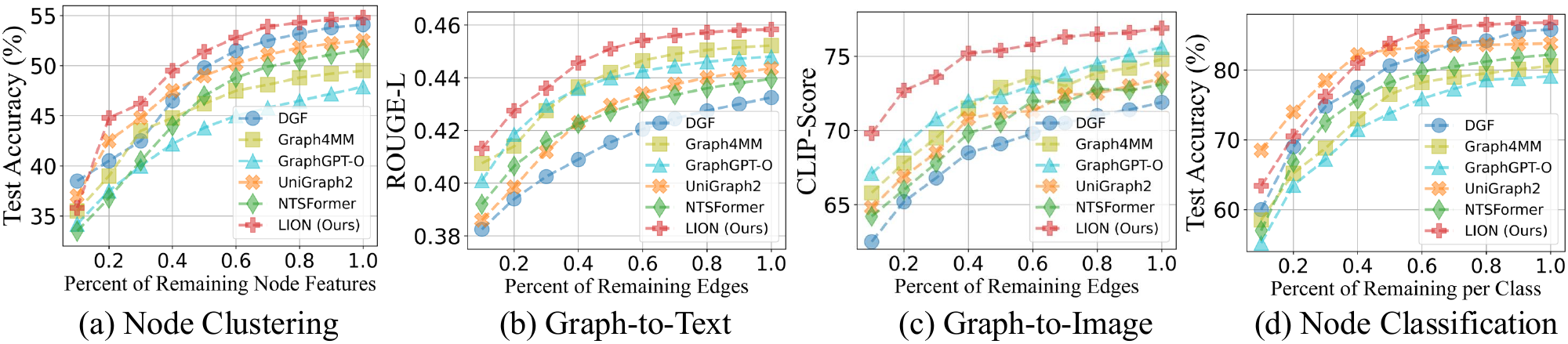}
  \vspace{-0.5cm}
  \caption{Performance of LION under different sparsity levels on the Grocery dataset.}
  \vspace{-0.2cm}
  \label{fig: robust_sparsity}
\end{figure*}

\begin{table}[t]
\setlength{\abovecaptionskip}{-0.01cm}
\centering
\caption{Epoch-Batch efficiency on Goodreads (Classification).}
\label{tab: efficiency}
\footnotesize 
\renewcommand{\arraystretch}{1.1}
\resizebox{\linewidth}{!}{
\setlength{\tabcolsep}{1.5mm}{
\begin{tabular}{l!{\vrule width 0.1pt}c!{\vrule width 0.1pt}c!{\vrule width 0.1pt}c!{\vrule width 0.1pt}c!{\vrule width 0.1pt}c}
\hline\thickhline
\rowcolor{gray!80}
\textbf{\textcolor{white}{Method}} & \textbf{\textcolor{white}{Pre-process (s)}} & \textbf{\textcolor{white}{E-Train. (s)}} & \textbf{\textcolor{white}{E-Infer. (s)}} & \textbf{\textcolor{white}{B-GPU Mem.}} & \textbf{\textcolor{white}{Param.}} \\
\hline
MLaGA       & - & 112.4$_{\pm 1.05}$ & 45.2$_{\pm 0.93}$ & 22.8G & 125M \\
\rowcolor{gray!10} GraphGPT-O  & - & 125.6$_{\pm 1.70}$ & 48.9$_{\pm 1.45}$ & 24.5G & 140M \\
Graph4MM    & 45.2$_{\pm 1.03}$ & 28.5$_{\pm 0.85}$ & 10.4$_{\pm 0.42}$ & 17.2G & 96M \\
\rowcolor{gray!10} InstructG2I & - & 168.2$_{\pm 2.50}$ & 75.6$_{\pm 1.10}$ & 32.1G & 180M \\
\hline
DMGC        & - & 9.2$_{\pm 0.15}$ & 4.1$_{\pm 0.08}$ & 3.8G & 2.4M \\
\rowcolor{gray!10} DGF         & - & 10.8$_{\pm 0.20}$ & 3.6$_{\pm 0.10}$ & 4.1G & 2.6M \\
MIG-GT      & 63.7$_{\pm 1.16}$ & 5.4$_{\pm 0.12}$ & 2.5$_{\pm 0.05}$ & 2.2G & 1.8M \\
\rowcolor{gray!10} NTSFormer   & 75.8$_{\pm 1.25}$ & 4.1$_{\pm 0.10}$ & 1.9$_{\pm 0.04}$ & 2.5G & 2.1M \\
{LION (Ours)} & 92.4$_{\pm 1.49}$ & {3.2}$_{\pm 0.09}$ & {1.2}$_{\pm 0.04}$ & {2.9G} & {1.5M} \\
\hline\thickhline
\end{tabular}
}}
\end{table}

\textbf{AHA Module}.
    First of all, we clarify the modality fusion mechanism implemented from a holographic perspective involving both energy and scale (Sec.~\ref{sec: Adaptive Holographic Aggregation}).
    In Fig.~\ref{fig: aha_fusion}, modality channels are primarily filtered via energy-initialized adaptive gating to select the informative and beneficial interaction subspaces. 
    Then, propagation depth determines the manner in which multi-scale structural contexts are integrated within these interaction channels.
    Based on this, we reveal that LION dynamically assigns varying importance to channels (energy) and depth (scale) according to the specific scenarios. 
    For instance, a distinct modality-aware preference emerges where the model assigns significantly higher weights to text-related subspaces on Goodreads while shifting its focus toward image-related subspaces on Ele-Fashion. 
    Meanwhile, their optimal receptive fields also vary across these datasets to ensure optimal modality fusion. 
    This adaptive behavior is consistent with the heterogeneous nature of MAGs.
    Some domains depend more heavily on textual semantics, such as book summaries and reviews, whereas fashion-related graphs often require visual compatibility cues.
    A fixed fusion rule would obscure such domain-specific preferences, but the energy-scale design allows LION to allocate capacity to the most informative channels and propagation ranges.

\subsection{Robustness Analysis}
\label{sec: Robustness Analysis}
    To answer \textbf{Q4}, we investigate the robustness of LION in sparse scenarios.
    This stems from the fact that MAGs often suffer from data incompleteness.
    As shown in Fig.~\ref{fig: robust_sparsity}, LION demonstrates remarkable robustness across all sparsity scenarios. 
    Meanwhile, LION consistently outperforms other methods in all downstream tasks. 
    This superiority is attributed to the Clifford geometric manifold, which can mitigate missing data through high-order graph propagation. 
    This indicates that our alignment-then-fusion mechanism effectively exploits the complementary nature of topology and modality, allowing the model to learn high-quality representations even when raw information is severely corrupted.
    In sparse scenarios, local attributes alone become unreliable because either modality content or neighborhood evidence may be partially absent.
    CGP alleviates this issue by propagating curvature-aware context from broader graph neighborhoods, while AHA filters noisy channels before final prediction.
    As a result, the model remains stable even when the input graph provides incomplete or imbalanced multimodal evidence.

\begin{figure}[t]
  \centering
  \includegraphics[width=\linewidth]{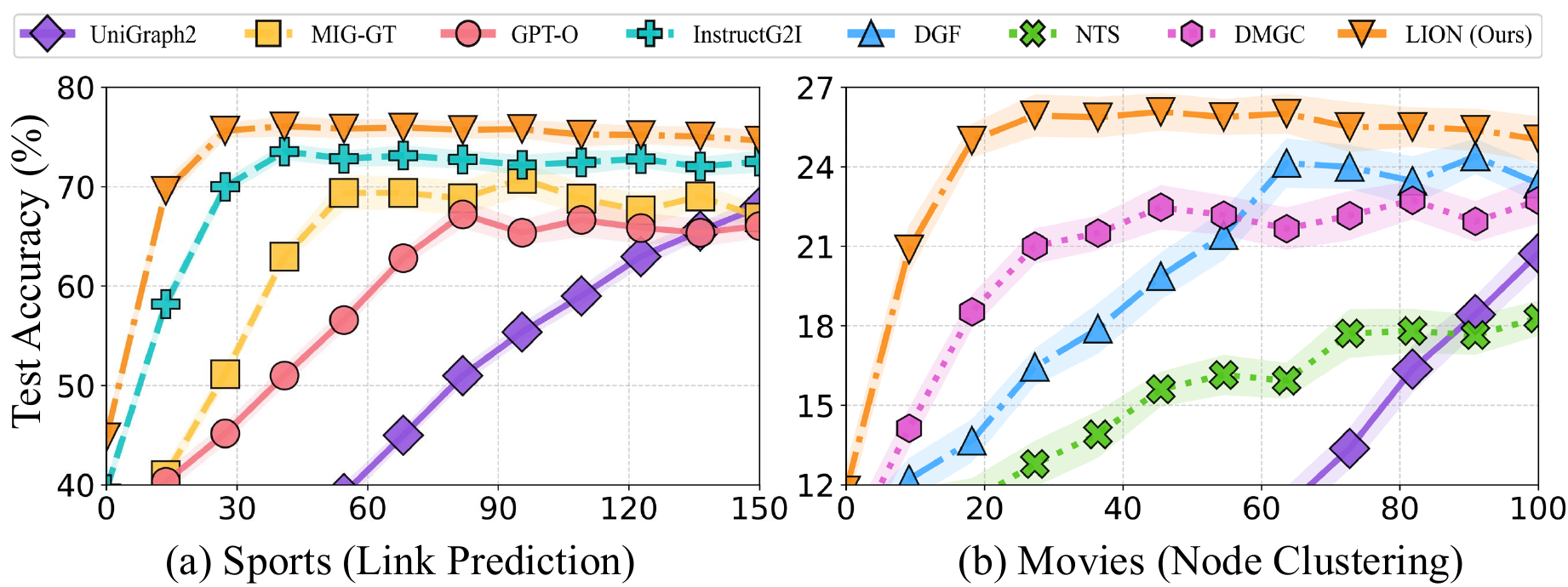}
  \caption{The convergence efficiency curves.}
  \label{fig: convergence}
\end{figure}

\subsection{Running Efficiency}
\label{sec: Running Efficiency}
    To answer \textbf{Q5}, we provide the computational overhead and convergence curves of LION.
    As summarized in Table~\ref{tab: efficiency}, LION demonstrates superior efficiency compared to other baselines. 
    This advantage is primarily attributed to our decoupled graph neural paradigm, which separates intensive geometric computations from the training loop. 
    Specifically, the construction of the spatial rotor and geometric potential within the CGP is relegated to a one-time pre-processing.
    Consequently, the training phase focuses exclusively on the AHA and avoids the need for recursive neighborhood expansion.
    Fig.~\ref{fig: convergence} illustrates that LION exhibits a rapid convergence rate, consistently reaching peak performance in significantly fewer epochs than other baselines. 
    This empirical training curve is consistent with using stable cached CGP tokens, while Theorem~\ref{theorem2} concerns spectral contraction of the fixed CGP propagation operator rather than the convergence rate of the downstream optimizer. 
    By reducing repeated structural computation in the trainable loop, LION facilitates efficient adaptation.
    The convergence behavior also reflects the benefit of decoupling propagation from trainable aggregation.
    Since CGP produces stable multiscale tokens before optimization, the trainable module does not need to repeatedly discover structural context from scratch.
    Instead, it focuses on selecting informative interaction channels and receptive fields, which explains the lower epoch-level training and inference cost in Table~\ref{tab: efficiency}.

\section{Conclusion}
    With the rapid advancement of multimodal domains, MAG has emerged as a critical frontier for enhancing graph representation and downstream utility. 
    In this work, we propose LION, which unifies topology with modality by constructing the geometric manifold grounded in Clifford algebra. 
    Our theoretical analysis establishes bounded geometric sensitivity, conditional spectral contraction of CGP, and stability bounds for AHA.
    Extensive experiments across a broad spectrum of supervised and unsupervised tasks demonstrate that LION significantly outperforms SOTA baselines, confirming its superiority in both graph and
    modality tasks.
    Although the manifold dimension scales exponentially with the number of modalities, this overhead remains marginal in practice, as the modality count is typically limited.
    Future work includes the integration of LION into broader applications such as Graph QA. 
    Furthermore, enhancing the utility of Multimodal Large Models within LION to exploit cross-modal semantics under topology priors is also a compelling direction.

\clearpage
\appendices

\section{Proof of Geometric Stability in Theorem 1}
\label{appendix: Proof of Geometric Stability}
    This section proves the stability statement in Theorem 1 using the same initialization, potential, and stabilized rotor as the main paper.

\begin{proof}
    For modality $k$, let $\bar{x}_u^{(k)}=g_k(x_u^{(k)})\in\mathbb R^d$ and $\hat{x}_u^{(k)}=\bar{x}_u^{(k)}/\max(\|\bar{x}_u^{(k)}\|_2,\epsilon_x)$ with fixed $\epsilon_x>0$. Assume each $g_k$ is Lipschitz on the bounded input domain. The initialization is
    \begin{equation}
    \mathbf H_u^{(0)}=\frac1{\sqrt K}\sum_{k=1}^K\hat{x}_u^{(k)}e_k,\qquad \|\mathbf H_u^{(0)}\|_{\mathcal Cl}\le1.
    \end{equation}
    Stabilized normalization is Lipschitz, and the orthogonal Clifford lift is linear; hence for some $L_H<\infty$,
    \begin{equation}
    \|\mathbf H-\mathbf H'\|_{\mathcal Cl}\le L_H\|\mathbf X-\mathbf X'\|_{\mathcal Cl}\equiv\delta_H.
    \end{equation}
    For an edge, the geometric product is bilinear. Because both lifted representations have norm at most one,
    \begin{equation}
    \|\mathbf H_u\mathbf H_v-\mathbf H'_u\mathbf H'_v\|_{\mathcal Cl}\le2\delta_H,
    \end{equation}
    and the Grade-0 and Grade-2 projections are non-expansive. Let $\mathcal S_{uv}=\langle\mathbf H_u\mathbf H_v\rangle_0$ and $\mathcal B_{uv}=\langle\mathbf H_u\mathbf H_v\rangle_2$. The rotor used in the main paper is
    \begin{equation}
    \begin{aligned}
    \widehat{\mathcal B}_{uv}&=\frac{\mathcal B_{uv}}{\sqrt{\|\mathcal B_{uv}\|^2+\epsilon_r^2}},\\
    \theta_{uv}&=\operatorname{atan2}(\|\mathcal B_{uv}\|,\|\mathcal S_{uv}\|+\epsilon_r),\\
    \mathcal R_{uv}&=\exp(-\theta_{uv}\widehat{\mathcal B}_{uv}/2).
    \end{aligned}
    \end{equation}
    With fixed $\epsilon_r>0$, both stabilized normalization and $\operatorname{atan2}$ are Lipschitz on the bounded domain, including $\|\mathcal B_{uv}\|\rightarrow0$; the exponential is Lipschitz on the resulting compact set. Thus $\|\mathcal R-\mathcal R'\|\le L_R\delta_H$. The potential in the main paper contains a positive stabilizer $\epsilon$ and is likewise Lipschitz, $\|\Phi-\Phi'\|\le L_\Phi\delta_H$.

    Writing $\mathcal A_{\mathcal G}=\mathbf A\circ(\Phi,\mathcal R)$, bounded degree and the product rule give
    \begin{equation}
    \|\mathcal A_{\mathcal G}-\mathcal A'_{\mathcal G}\|_{\mathcal Cl}
    \le C_1\delta_H+C_2\|\mathbf A-\mathbf A'\|_{\mathcal Cl}.
    \end{equation}
    Combining the two inequalities yields
    \begin{equation}
    \begin{aligned}
    &\|\mathbf H-\mathbf H'\|_{\mathcal Cl}+\|\mathcal A_{\mathcal G}-\mathcal A'_{\mathcal G}\|_{\mathcal Cl}\\
    &\quad\le K_{\rm map}(\|\mathbf X-\mathbf X'\|_{\mathcal Cl}+\gamma\|\mathbf A-\mathbf A'\|_{\mathcal Cl}).
    \end{aligned}
    \end{equation}
    where $K_{\rm map}$ and $\gamma$ depend on the encoder/projection bounds, graph degree, and fixed stabilizers. This proves Theorem 1 without requiring the bi-vector norm to be bounded away from zero.
\end{proof}

\section{Proof of Spectral Evidence in Theorem 2}
\label{appendix: Proof of Spectral Evidence}
    We analyze exactly the normalized CGP update defined in the main paper. During one cached propagation stage, assume an undirected graph, symmetric potentials $\Phi_{uv}=\Phi_{vu}$, reciprocal isometric transports $\mathcal T_{vu}=\mathcal T_{uv}^{-1}$, and fixed edge operators. Define $\bar\Phi_{uu}=1$, $\bar\Phi_{uv}=\Phi_{uv}$, $d_u=\sum_{k\in\mathcal N(u)\cup\{u\}}\bar\Phi_{uk}$, and $\tilde\Phi_{uv}=\bar\Phi_{uv}/d_u$. Then the main-paper propagation is
    \begin{equation}
    \mathbf H_u^{(l)}=\sum_{v\in\mathcal N(u)\cup\{u\}}\tilde\Phi_{uv}\,\mathcal T_{uv}(\mathbf H_v^{(l-1)})
    \equiv(\mathcal P_{\mathcal G}\mathbf H^{(l-1)})_u.
    \end{equation}
    The corresponding unnormalized connection Dirichlet energy is
    \begin{equation}
    \mathbb E_{\rm Dir}(\mathbf H)=\frac12\sum_{(u,v)\in\mathcal E}\Phi_{uv}\|\mathbf H_u-\mathcal T_{uv}(\mathbf H_v)\|_{\mathcal Cl}^2.
    \end{equation}
    Symmetry, reciprocity, and isometry make the associated connection Laplacian positive semidefinite. After degree normalization, $\mathcal P_{\mathcal G}$ is the connection averaging operator associated with this Laplacian (including the stated self loop) and is self-adjoint under the degree-weighted inner product $\langle\mathbf X,\mathbf Y\rangle_D=\sum_u d_u\langle\mathbf X_u,\mathbf Y_u\rangle_{\mathcal Cl}$. Let $\Pi$ be the corresponding orthogonal projection onto the unit-eigenspace of $\mathcal P_{\mathcal G}$. Decompose
    $\mathbf H^{(0)}=\Pi\mathbf H^{(0)}+(\mathbf I-\Pi)\mathbf H^{(0)}$. Because $\mathcal P_{\mathcal G}\Pi=\Pi$,
    \begin{equation}
    \mathbf H^{(l)}-\Pi\mathbf H^{(0)}=\mathcal P_{\mathcal G}^{l}(\mathbf I-\Pi)\mathbf H^{(0)}.
    \end{equation}
    If the spectral radius of $\mathcal P_{\mathcal G}$ on the complementary invariant subspace is $\rho<1$, operator submultiplicativity gives
    \begin{equation}
    \|\mathbf H^{(l)}-\Pi\mathbf H^{(0)}\|_{D,\mathcal Cl}
    \le\rho^l\|\mathbf H^{(0)}-\Pi\mathbf H^{(0)}\|_{D,\mathcal Cl},
    \end{equation}
    This is the contraction statement in Theorem 2. The assumptions are explicit: the result applies to the fixed reciprocal transport operator used for cached CGP and does not claim the same spectral rate for arbitrary time-varying edge operators. Because each $\mathcal T_{uv}$ is a rotor sandwich map, it is isometric and does not collapse the orthogonal Clifford basis merely by transport.

\section{Proof of AHA Stability and Consensus Bound in Theorem 3}
\label{appendix: Proof of Holographic Reconstruction}

This section proves the AHA stability inequalities in Theorem~3.
Importantly, the proof only uses the linear part of the affine
Clifford projection and does not assume that the learned consensus
profile is a downstream-optimal representation.

\begin{proof}
Let the affine Clifford projection used in the main paper be
\begin{equation}
    f_{\mathrm{cll}}(\mathbf X)
    =
    \mathbf W_{\mathrm{out}}\mathbf X
    +
    \mathbf b_{\mathrm{out}},
\end{equation}
and assume
\begin{equation}
    \|\mathbf W_{\mathrm{out}}\|_{2\rightarrow2}
    \le \omega.
\end{equation}

Define
\begin{equation}
\begin{aligned}
    \mathbf Z_u
    &=
    f_{\mathrm{cll}}
    \left(
    \sum_{l=0}^{L}
    \beta_{u,l}
    \widetilde{\mathbf H}_u^{(l)}
    \right),\\
    \mathbf Z_u^{\mathrm{raw}}
    &=
    f_{\mathrm{cll}}
    \left(
    \sum_{l=0}^{L}
    \beta_{u,l}
    \mathbf H_u^{(l)}
    \right).
\end{aligned}
\end{equation}

Since the bias term cancels when two affine outputs are subtracted,
we obtain
\begin{equation}
\begin{aligned}
\mathbf Z_u-\mathbf Z_u^{\mathrm{raw}}
&=
\mathbf W_{\mathrm{out}}
\sum_{l=0}^{L}
\beta_{u,l}
\left(
\widetilde{\mathbf H}_u^{(l)}
-
\mathbf H_u^{(l)}
\right).
\end{aligned}
\end{equation}
Using
$\widetilde{\mathbf H}_u^{(l)}
=
\boldsymbol\alpha_u^{(l)}
\odot\mathbf H_u^{(l)}$,
the operator-norm inequality and triangle inequality yield
\begin{equation}
\begin{aligned}
\|
\mathbf Z_u-\mathbf Z_u^{\mathrm{raw}}
\|_2
&\le
\omega
\left\|
\sum_{l=0}^{L}
\beta_{u,l}
\left(
\widetilde{\mathbf H}_u^{(l)}
-
\mathbf H_u^{(l)}
\right)
\right\|_{\mathcal Cl}\\
&\le
\omega
\sum_{l=0}^{L}
\beta_{u,l}
\left\|
(\mathbf 1-\boldsymbol\alpha_u^{(l)})
\odot
\mathbf H_u^{(l)}
\right\|_{\mathcal Cl}.
\end{aligned}
\end{equation}
This proves the first inequality and quantifies the representation
change introduced by the soft component gate.
No assumption is made that an attenuated component is noise.

For the consensus bound, because
$\sum_{l=0}^{L}\beta_{u,l}=1$,
\begin{equation}
\begin{aligned}
&\mathbf Z_u
-
f_{\mathrm{cll}}
\left(
\mathbf H_u^{\mathrm{ctx}}
\right)\\
&=
\mathbf W_{\mathrm{out}}
\left[
\sum_{l=0}^{L}
\beta_{u,l}
\widetilde{\mathbf H}_u^{(l)}
-
\mathbf H_u^{\mathrm{ctx}}
\right]\\
&=
\mathbf W_{\mathrm{out}}
\sum_{l=0}^{L}
\beta_{u,l}
\left(
\widetilde{\mathbf H}_u^{(l)}
-
\mathbf H_u^{\mathrm{ctx}}
\right).
\end{aligned}
\end{equation}
Therefore,
\begin{equation}
\begin{aligned}
\left\|
\mathbf Z_u
-
f_{\mathrm{cll}}
(\mathbf H_u^{\mathrm{ctx}})
\right\|_2
&\le
\omega
\left\|
\sum_{l=0}^{L}
\beta_{u,l}
\left(
\widetilde{\mathbf H}_u^{(l)}
-
\mathbf H_u^{\mathrm{ctx}}
\right)
\right\|_{\mathcal Cl}\\
&\le
\omega
\sum_{l=0}^{L}
\beta_{u,l}
\left\|
\widetilde{\mathbf H}_u^{(l)}
-
\mathbf H_u^{\mathrm{ctx}}
\right\|_{\mathcal Cl}.
\end{aligned}
\end{equation}

This proves the second inequality.
The result is a consensus-deviation bound and does not identify
$\mathbf H_u^{\mathrm{ctx}}$ with the unknown downstream-optimal
representation.
\end{proof}

\section{Theoretical Foundations of LION}
    The preceding analyses establish three scoped properties of LION. Theorem 1 bounds the sensitivity of the Clifford construction under bounded encoders/projections and fixed numerical stabilizers. Theorem 2 gives a conditional spectral contraction result for the same normalized, fixed reciprocal CGP operator used in the main paper. Theorem 3 bounds the representation change introduced by AHA gating and its deviation from the learned multiscale consensus. These results support stability and operator consistency; they do not assume that attenuated grades are necessarily noise, that the consensus is a downstream optimum, or that the fixed-operator spectral rate describes the downstream neural optimizer.

\section{Algorithm and Complexity Details}
\label{appendix: Algorithm and Complexity Details}
    We provide additional details for the algorithmic flow and computational complexity of LION.
    The data flow begins with the initialization of the Clifford geometric manifold.
    Given a multimodal-attributed graph $\mathcal{G}$, the modality-oriented isomorphic lifting map projects raw scalar features $\mathbf{X}$ into the Grade-1 subspace of the Clifford algebra $\mathcal{C}l_K$.
    Assigning each modality to a unique orthogonal basis vector $e_k$ preserves the independence of modality axes while unifying them in the multi-vector representation $\mathbf{H}^{(0)}$.
    Topology modeling is then performed by computing the geometric product between connected nodes, whose Grade-0 and Grade-2 projections provide the same-modality compatibility and oriented cross-modality discrepancy statistics defined in the main paper.
    These edgewise statistics instantiate the topology-oriented geometric potential $\Phi$ and spatial rotor $\mathcal{R}$, which jointly encode semantic curvature on the Clifford manifold.
    Guided by these operators, CGP conducts training-free high-order propagation through potential-gated parallel transport, producing aligned curvature-adaptive tokens $\{\mathbf{H}^{(1)},\dots,\mathbf{H}^{(L)}\}$.

    For the two-modality case, the text and visual attributes may have different raw dimensions $d_t$ and $d_i$; modality-specific projections first map both to the shared dimension $d$ before the Grade-1 lift. The $Cl_2$ storage then uses the basis slots $\{1,e_1,e_2,e_{12}\}$, with the normalized text and image features assigned to $e_1$ and $e_2$.
    The scalar and bi-vector quantities used by CGP are edgewise geometric-product statistics; they are not four ordered modality-pair channels stored one-to-one in these basis slots. Accordingly, the T-T, T-I, I-T, and I-I labels in the visualization denote diagnostic ordered-pair contributions rather than Clifford basis labels.
    AHA quantifies the information density of the resulting Clifford token components, applies a soft learnable gate, and fuses different propagation depths through resonance scores against the consensus profile $\mathbf{H}^{\text{ctx}}$. The final representation is obtained by projecting the filtered and scale-weighted Clifford multi-vector back into Euclidean space.

    We further detail the time-space complexity.
    Let $N=|\mathcal{V}|$ and $M=|\mathcal{E}|$ denote the numbers of nodes and edges, $L$ the propagation depth, and $D=2^K d$ the Clifford multi-vector dimension for $K$ modalities.
    Manifold initialization via the lifting map requires $\mathcal{O}(N\cdot D)$.
    Constructing the geometric potential and spatial rotor requires evaluating the geometric product on each edge, leading to $\mathcal{O}(M\cdot 2^K\cdot D)$. For $K>2$, the scalar component aggregates same-modality compatibility and the $\binom{K}{2}$ Grade-2 coefficients encode oriented pairwise cross-modality discrepancies as defined in the main paper; higher grades are available algebraically but are not empirically evaluated here.
    Curvature-adaptive high-order propagation costs $\mathcal{O}(L\cdot M\cdot D)$.
    Since CGP is training-free and parameter-independent, this phase is executed once and cached.
    During training, LION optimizes only AHA: energy-aware grade filtering and scale-aware resonance fusion both scale as $\mathcal{O}(L\cdot N\cdot D)$.
    Therefore, the total computational cost can be summarized as $\mathcal{O}(L(M\cdot D+N\cdot 2^K\cdot D))$, while remaining linear in graph size for fixed $K$ and $L$.
    The cached multiscale propagated features require $\mathcal{O}(L\cdot N\cdot D)$ memory, and the learnable projection, gating, and attention parameters require $\mathcal{O}(D^2)$ space.
    This decoupled design avoids recursive neighborhood expansion during back propagation and enables scalable mini-batch training.

\section{Dataset Details}
\label{appendix: Dataset Details}
    All nine benchmarks used in this paper provide text and image attributes; although the Clifford formulation is defined for $K$ modality axes, empirical claims in this work are restricted to the $K=2$ text-image setting. Dataset-task assignments follow Table I and the headers of the corresponding result tables in the main paper. Graph construction uses only benchmark relations or source metadata available independently of the downstream target split; class labels, held-out links, retrieval correspondences, and generation references are excluded from edge construction to avoid target leakage.

    \textbf{RedditS}~\cite{desai1_RedditS}
    is a social network from Reddit where nodes represent posts and edges denote threading relationships (e.g., comments and replies). 
    Textual features are encoded from titles and body content, while visual features are extracted from embedded images.
    This dataset is used for node clustering in the reported graph-task comparison.

    \textbf{Movies}~\cite{ni2019_Grocery_Cloth_Ele_Movies_Sports} 
    is sourced from Amazon's Movies and TV category. 
    Nodes correspond to DVD/Blu-ray products, and edges reflect consumer co-purchasing behavior. 
    Node attributes include textual plot synopses and customer reviews, alongside visual features derived from official cover art. 
    This dataset is used for node classification in the reported graph-task comparison.

    \textbf{Grocery}~\cite{ni2019_Grocery_Cloth_Ele_Movies_Sports} 
    originates from Amazon's Grocery and Gourmet Food segment. 
    Edges indicate complementary purchasing habits derived from "also-bought" metadata. 
    Textual features are encoded from product titles and nutritional descriptions, while visual features are extracted from packaging images. 
    This dataset is used for node clustering in the reported graph-task comparison.

    \textbf{SemArt}~\cite{garcia2018_SemArt} 
    is a fine-art dataset where nodes represent paintings and edges are established based on shared metadata such as artist, period, or school. 
    Node features include expert historical commentary and stylistic visual attributes from digital images. 
    This dataset is used for Graph-to-Image (G2Image) in the reported modality-task comparison, with its paired textual and visual attributes retained as node modalities.

    \textbf{Flickr30k}~\cite{plummer2015_Flickr30k} 
    is a canonical image-text reasoning dataset. 
    In the OpenMAG setting, we follow the released benchmark graph relations, which are fixed before the downstream split; held-out target captions used for evaluation are not used to construct edges. 
    This dataset is utilized for Graph-to-Text (G2Text) tasks, evaluating the model's ability to generate descriptive captions from graph-structured visual-textual context.

    \textbf{Sports}~\cite{hou2024_Cloth_Ele_Sports, ni2019_Grocery_Cloth_Ele_Movies_Sports} 
    is an Amazon-based graph of athletic gear where edges denote functional complementarity.
    It incorporates technical specifications (text) and product images (visual) to capture design and utility. 
    This dataset is primarily utilized for link prediction, aiming to forecast potential product associations by modeling the geometric compatibility between sports-related modalities.

    \textbf{Ele-fashion}~\cite{ni2019_Grocery_Cloth_Ele_Movies_Sports,hou2024_Cloth_Ele_Sports} 
    is a heterogeneous graph merging Amazon's Electronics and Fashion categories. Nodes are connected via cross-category co-purchasing links, revealing latent consumer preferences across disparate domains.
    Features combine technical specs with style descriptions and product imagery. 
    This dataset is used for modality retrieval in the reported modality-task comparison.

    \textbf{Cloth}~\cite{ni2019_Grocery_Cloth_Ele_Movies_Sports,hou2024_Cloth_Ele_Sports} 
    is a large-scale Amazon fashion graph linking items via visual compatibility or co-purchase history. 
    It utilizes material compositions (text) and high-resolution model images (visual).
    The dataset is used for link prediction in the reported graph-task comparison, where visual and textual compatibility provides a test for modality fusion and alignment.

    \textbf{Goodreads}~\cite{wan2018_Goodreads_NC,wan2019_Goodreads_NC} 
    is a book review graph where edges represent co-shelving relationships. 
    Node features are derived from book summaries, editorial reviews, and cover art.
    The dataset is used for node classification in the reported graph-task comparison, requiring the synthesis of aesthetic cover cues with the semantic depth of textual summaries.

\section{Baselines Details}
\label{appendix: Baselines Details}

    \textbf{GCN}~\cite{kipf2016gcn} utilizes a localized first-order approximation of spectral graph convolutions based on the renormalization trick.
    It scales linearly with the number of edges and learns hidden representations that jointly capture graph topology structure and node attribute features via a layer-wise propagation rule.
    
    \textbf{GCNII}~\cite{chen2020gcnii} extends the conventional GCN by incorporating two key techniques: initial residual connections and identity mapping.
    These mechanisms effectively simulate lazy random walks and are theoretically and empirically shown to alleviate the over-smoothing issue, allowing the model to stack deep architectures without performance degradation.
    
    \textbf{GAT}~\cite{velivckovic2017gat} introduces masked self-attention layers to assign learnable importance scores to neighbors during aggregation.
    This strategy enables the model to handle anisotropic graphs by implicitly assigning varying weights to different nodes, thereby capturing complex local structural patterns without relying on a fixed, predefined graph Laplacian.
    
    \textbf{GATv2}~\cite{brody2021gatv2} implements a dynamic graph attention mechanism that strictly improves upon the static attention of the original GAT.
    By modifying the order of operations in the attention scoring function, it achieves universal approximation capability, offering enhanced representation expressiveness and increased robustness against graph noise and irrelevant connections.
    
    \textbf{MMGCN}~\cite{hu2021mmgcn} is a pioneering framework for micro-video recommendation that explicitly models user preferences across visual, acoustic, and textual modalities.
    It constructs separate modality-specific bipartite graphs and aggregates high-order connectivity within each graph before combining them through a structured fusion layer, effectively capturing the user-item interactions inherent in each sensory channel.
    
    \textbf{MGAT}~\cite{tao2020mgat} employs a gated attention mechanism within parallel multimodal interaction graphs for recommendation.
    It adaptively identifies the importance of specific modalities to disentangle granular personal interests, effectively acting as a denoising filter to reduce the influence of noisy or conflicting multimodal signals during preference learning.
    
    \textbf{MLaGA}~\cite{fan2025mlaga} enables Large Language Models (LLMs) to reason over MAGs via a coherent two-stage alignment strategy.
    It first aligns visual and textual features with the graph structure through contrastive graph pre-training, and subsequently performs instruction tuning to seamlessly integrate graph connectivity priors into the LLM's generative reasoning process.
    
    \textbf{GraphGPT-O}~\cite{fang2025graphgpt_o} is a comprehensive multimodal LLM designed for joint comprehension and generation tasks on MAGs.
    It addresses the challenge of encoding non-Euclidean dependencies and scalability by employing personalized PageRank sampling and a hierarchical aligner equipped with both node-level and graph-level Q-Formers to bridge structural tokens with the LLM semantic space.
    
    \textbf{Graph4MM}~\cite{ning2025graph4mm} integrates multi-hop structural information directly into the self-attention mechanism via a hop-diffused attention strategy.
    It utilizes a specialized MM-QFormer for principled cross-modal fusion, demonstrating that treating graph topology as a guided interaction modality significantly outperforms approaches that treat the graph merely as an auxiliary input feature.
    
    \textbf{InstructG2I}~\cite{jin2024instructg2i} is a graph-conditioned diffusion model specifically designed for MAGs.
    It utilizes semantic neighbor sampling to construct contextual prompts and employs a Graph-QFormer to encode these prompts, offering fine-grained controllability over the generative process through a novel graph classifier-free guidance mechanism.
    
    \textbf{DMGC}~\cite{guo2025DMGC} addresses complex hybrid neighborhood patterns by explicitly disentangling the graph into cross-modality homophily-enhanced and modality-specific heterophily-aware components.
    It utilizes a dual-frequency fusion mechanism, which acts as coupled low-pass and high-pass filters to simultaneously capture both intra-class commonalities (smoothness) and inter-class distinctions (boundaries).
    
    \textbf{DGF}~\cite{zheng2025dgf} proposes a cross-contrastive clustering framework that utilizes a dual graph filtering scheme to systematically denoise features extracted from MAG.
    It employs a tri-cross contrastive objective spanning modalities, topology neighborhoods, and semantic communities to learn discriminative clustering representations robust to outliers.
    
    \textbf{MIG-GT}~\cite{hu2025mig_gt} utilizes modality-independent GNNs with adaptive receptive fields to accommodate the unique propagation requirements and noise levels of distinct modalities.
    To complement local aggregation, it integrates a sampling-based global transformer, enabling the model to capture long-range semantic dependencies and global context that standard message passing often fails to preserve.
    
    \textbf{NTSFormer}~\cite{hu2025ntsformer} introduces a self-teaching graph transformer tailored for isolated cold-start node classification.
    It employs a stochastic cold-start attention mask to supervise a student prediction (based solely on self-information) with a teacher prediction (which is neighbor-aware), thereby ensuring robust performance and generalization even when structural connections or modal attributes are partially missing.
    
    \textbf{UniGraph2}~\cite{he2025unigraph2} is a cross-domain graph foundation model that integrates diverse modalities into a unified embedding space.
    It leverages frozen pre-trained encoders and a mixture-of-experts (MoE) module for scalable alignment, followed by a universal GNN for structural aggregation, facilitating the learning of generalizable representations that transfer effectively across various downstream tasks and domains.

\paragraph{Baseline adaptation and tuning protocol.}
    We retain each baseline's released backbone and modality preprocessing whenever available and evaluate it on the same dataset split and downstream metric. If a released implementation lacks the required task head, only the task head is adapted while the representation module is kept unchanged. Hyperparameters and early stopping are selected using validation data only, never the test set. Shared task-side dimensions and stopping rules are used when the original method permits them; otherwise, the authors' recommended method-specific settings are retained. Graph edges follow the target-independent construction protocol described in Appendix F.

\section{Evaluation Protocols}
\label{appendix: Evaluation Protocols}
    To ensure a rigorous and standardized evaluation across diverse MAG learning models, we implement specific experimental configurations for three primary graph-centric downstream tasks: supervised node classification, supervised link prediction, and unsupervised node clustering. 
    For node classification and node clustering, the learning rate is established at $5 \times 10^{-3}$ with a batch size of 512 and a weight decay of $1 \times 10^{-5}$. 
    While 100 training epochs are sufficient for convergence in the node classification task, node clustering requires an extended duration of 500 epochs to stabilize the underlying self-supervised objectives. 
    In contrast, the link prediction task utilizes an adjusted learning rate of $1 \times 10^{-3}$ and a significantly larger batch size of 2048 to facilitate the efficient processing of extensive edge-pair samples. 
    To maintain architectural consistency and ensure a fair comparison, we standardize core parameters by employing CLIP-ViT-Large-Patch14 as the default frozen feature encoder, with feature dimensions unified at 768 across all tasks.
    For graph-task comparisons and ablations, results are obtained from five independent runs and reported as mean$\pm$standard deviation, matching entries marked with ``$\pm$'' in the main paper.
    
    \textbf{Node Classification} is the fundamental supervised task in graph learning. 
    Given a MAG, the model encodes each node into a low-dimensional embedding. 
    These representations are subsequently processed by a projection head and a Softmax layer to generate class probabilities. 
    Optimization is performed by minimizing the cross-entropy loss between the predicted distributions and ground-truth labels.
    Model performance is quantified using Accuracy (Acc) and the F1-score.
    
    Specifically, Acc measures the proportion of correctly predicted samples over the evaluation set: $\text{Acc} = \frac{1}{N} \sum_{i=1}^{N} \mathbb{I}(\hat{y}_i = y_i)$, where $N$ is the sample size, $y_i$ is the ground-truth, $\hat{y}_i$ is the prediction, and $\mathbb{I}(\cdot)$ is the indicator function.
    F1-score is the harmonic mean of Precision and Recall, providing a robust evaluation under class imbalance common in real-world graphs: $\text{F1} = 2 \left(\cdot \text{Precision} \cdot \text{Recall}\right)/\left(\text{Precision} + \text{Recall}\right)$.

    \textbf{Link Prediction}
    measures the model's capacity to infer missing or potential edges. 
    The model computes similarity scores between node pairs (e.g., via dot product) and assigns higher scores to true edges than to negative samples. 
    In multimodal settings, this requires aligning structural proximity with cross-modal semantic similarity. 
    Evaluation relies on ranking-based metrics: Mean Reciprocal Rank (MRR) and Hits@K.
    
    Specifically, MRR assesses the model's ability to prioritize the correct target at the top of a candidate list: $\text{MRR} = \frac{1}{|Q|} \sum_{i=1}^{|Q|} \frac{1}{\text{rank}_i}$, where $\text{rank}_i$ denotes the rank of the first correct result for query $i$.
    Hits@K reflects the recall performance at a fixed cutoff $K$, indicating the practical utility of the retrieval results: $\text{Hits@K} = \frac{1}{|Q|} \sum_{i=1}^{|Q|} \mathbb{I}(\text{rank}_i \leq K)$.

    \textbf{Node Clustering} evaluates representation quality in an unsupervised setting, following the protocol in \cite{guo2025DMGC}. 
    The model partitions nodes into semantic groups without label supervision. 
    This involves disentangling the graph into homophilous and heterophilous views, followed by a dual-frequency fusion of filtered signals. 
    The model is optimized via a joint objective comprising reconstruction, contrastive alignment, and clustering losses. 
    Model performance is quantified using Normalized Mutual Information (NMI) and Adjusted Rand Index (ARI).
    
    Specifically, NMI quantifies the mutual dependence between predicted clusters $C$ and ground-truth labels $Y$, independent of label permutations: $\text{NMI}(Y, C) = 2 \cdot I(Y, C)/\left(H(Y) + H(C)\right)$, where $I(\cdot)$ and $H(\cdot)$ denote mutual information and entropy, respectively.
    ARI evaluates clustering similarity by considering all sample pairs and adjusting for chance grouping: 
    \begin{equation}
        \begin{aligned}
            \text{ARI} = \frac{\sum_{ij} \binom{n_{ij}}{2} - [\sum_i \binom{a_i}{2} \sum_j \binom{b_j}{2}] / \binom{n}{2}}{ \frac{1}{2} [\sum_i \binom{a_i}{2} + \sum_j \binom{b_j}{2}] - [\sum_i \binom{a_i}{2} \sum_j \binom{b_j}{2}] / \binom{n}{2} },
        \end{aligned}
    \end{equation} 
    where $n_{ij}$ represents the overlap between ground-truth cluster $i$ and predicted cluster $j$.

    For the modality retrieval task, which evaluates the model's capability to search for relevant instances across image and text modalities within a shared latent space, we employ contrastive learning objectives integrated with a temperature scaling factor of $\tau = 0.07$. 
    These retrieval models are trained for 500 epochs using a learning rate of $1 \times 10^{-3}$ and a batch size of 256. 
    To effectively mitigate over-fitting and ensure the generalizability of the learned representations, we implement an early stopping mechanism with a patience period ranging from 10 to 25 epochs.
    For modality generation tasks, the configurations are specifically tailored to accommodate high-dimensional generative processes and complex many-to-many relationships within the graph. 
    In the Graph-to-Text (G2Text) task, we set the learning rate to $1 \times 10^{-3}$ with a weight decay of $1 \times 10^{-2}$ and a batch size of 8, training the model for 15 epochs.
    We utilize the Self-Attention with Embeddings (SA-E) strategy to sample four multimodal neighbors and employ Graph Neural Networks (GNNs) for structural position encoding. 
    The decoder backbone is powered by the pre-trained Facebook OPT-125M, and to ensure parameter efficiency during adaptation, we support both Prefix Tuning and Low-Rank Adaptation (LoRA) with a rank of $r=64$.
    In the Graph-to-Image (G2Image) task, we employ a learning rate of $1 \times 10^{-4}$ and a batch size of 16 for a duration of 20 epochs.
    Following the InstructG2I framework, we adopt Semantic Personalized PageRank (PPR)-based neighbor sampling, selecting between 0 and 6 informative neighbors to provide structural context. 
    The image resolution is standardized at 256 to condition the Stable Diffusion v1.5 backbone through a Graph Classifier-Free Guidance mechanism.
    
    To ensure experimental fairness and minimize performance variance, we standardize core architectural parameters across all task configurations. 
    Unless otherwise specified, we employ CLIP-ViT-L/14 as the default feature encoder, with node embedding dimensions unified at 768 across all downstream tasks. 
    All optimization processes are performed using the Adam or AdamW optimizers. 
    For graph-task comparisons and ablations, we use five independent runs and report mean$\pm$standard deviation. Modality retrieval and generation follow the corresponding benchmark protocols; entries without ``$\pm$'' in the main paper are single-run evaluations and are not presented as five-run averages.

    \textbf{Modality Retrieval} evaluates the model's capability to retrieve relevant instances across disparate modalities, specifically focusing on image-to-text and text-to-image retrieval within the LION. 
    Given a query from a source modality, the model projects both the query and the candidate set from the target modality into a unified, high-dimensional latent space for direct comparison. 
    Pairwise similarity scores are subsequently computed to rank candidates in descending order of predicted relevance. 
    This task serves as a rigorous assessment of the robustness of cross-modal alignment and the discriminative power of the learned representations for information retrieval. 
    To provide a comprehensive quantitative evaluation, model performance is measured using standard ranking metrics, including Mean Reciprocal Rank (MRR) and Hits@K.

    \textbf{Graph-to-Text (G2Text)} is a generative task requiring the model to synthesize natural language descriptions conditioned on graph-structured multimodal inputs. 
    Diverging from traditional one-to-one multimodal mappings, G2Text addresses complex many-to-many relationships where target nodes interact with diverse multimodal neighborhoods. 
    Following the MMGL framework~\cite{yoon2023mmgl}, the process involves three stages: 
    (i) Neighbor Encoding, which maps diverse modalities into a compatible embedding space; 
    (ii) Graph Structure Encoding, which captures topology context via GNNs or Laplacian position encodings; 
    (iii) Integration, where structure-aware signals are infused into a pre-trained LLM. 
    This task evaluates the faithful translation of structured graph information into coherent text.
    Model performance is evaluated by using standard NLG metrics, including BLEU-4, ROUGE-L, and CIDEr.
    
    Specifically, BLEU-4 evaluates lexical accuracy and fluency by calculating the geometric mean of modified $n$-gram precisions ($p_n$) up to length 4: 
    \begin{equation}
        \begin{aligned}
            \text{BLEU-4} = \text{BP} \cdot \exp\left(\sum_{n=1}^4 w_n \log p_n\right).
        \end{aligned}
    \end{equation}
    ROUGE-L measures sentence-level recall based on the longest common subsequence, ensuring the output covers the comprehensive information content of the ground truth: 
    \begin{equation}
        \begin{aligned}
            \text{ROUGE-L} = \frac{(1 + \beta^2) R_{lcs} P_{lcs}}{R_{lcs} + \beta^2 P_{lcs}}.
        \end{aligned}
    \end{equation}
    CIDEr measures the consensus between generated captions and human references using TF-IDF weighting, emphasizing the semantic importance and distinctiveness of the generated terms:
    \begin{equation}
    \begin{aligned}
        \text{CIDEr}_n(c, r) = \frac{1}{M} \sum_{i=1}^{M} \frac{g^n(c) \cdot g^n(r_i)}{|g^n(c)| |g^n(r_i)|}.
    \end{aligned}
    \end{equation}

    \textbf{Graph-to-Image (G2Image)} focuses on synthesizing images conditioned on MAG, requiring visual content to reflect both textual prompts and complex graph-based associations. 
    Following the InstructG2I framework~\cite{jin2024instructg2i}, the workflow consists of: 
    (i) Semantic PPR-based Sampling, which selects informative neighbors; 
    (ii) Graph-QFormer Encoding, which transforms neighbors into graph conditioning tokens; 
    (iii) Conditional Generation, where a latent diffusion model is guided by these tokens via a graph classifier-free guidance mechanism. 
    This task evaluates the model's ability to generate images consistent with the styles or categories defined by the graph structure.
    Model performance is quantified by using CLIP-Score and DINOv2-Score for instance-level consistency.
    
    Specifically, CLIP-Score quantifies cross-modal semantic consistency. 
    Building upon the textual encoder ($E_T$) and visual encoder ($E_I$) in pre-trained CLIP, this metric determines whether generated images faithfully preserve the semantic content of corresponding graph descriptions:
    \begin{equation}
    \begin{aligned}
        \text{CLIP-Score}(I, T) = \max\left(100 \cdot \cos(E_I(I), E_T(T)), 0\right).
    \end{aligned}
    \end{equation}
    DINOv2-Score assesses visual fidelity and structural consistency using feature embeddings from a pre-trained DINOv2 encoder, ensuring high perceptual quality and structural resemblance to reference samples:
    \begin{equation}
    \begin{aligned}
        \text{DINOv2-Score}(I_{\text{gen}}, I_{\text{ref}}) = \cos(\text{DINO}(I_{\text{gen}}), \text{DINO}(I_{\text{ref}})).
    \end{aligned}
    \end{equation}

\bibliographystyle{IEEEtran}
\bibliography{example_paper}

\begin{IEEEbiography}[{\includegraphics
[width=1in,height=1.25in,clip,keepaspectratio]{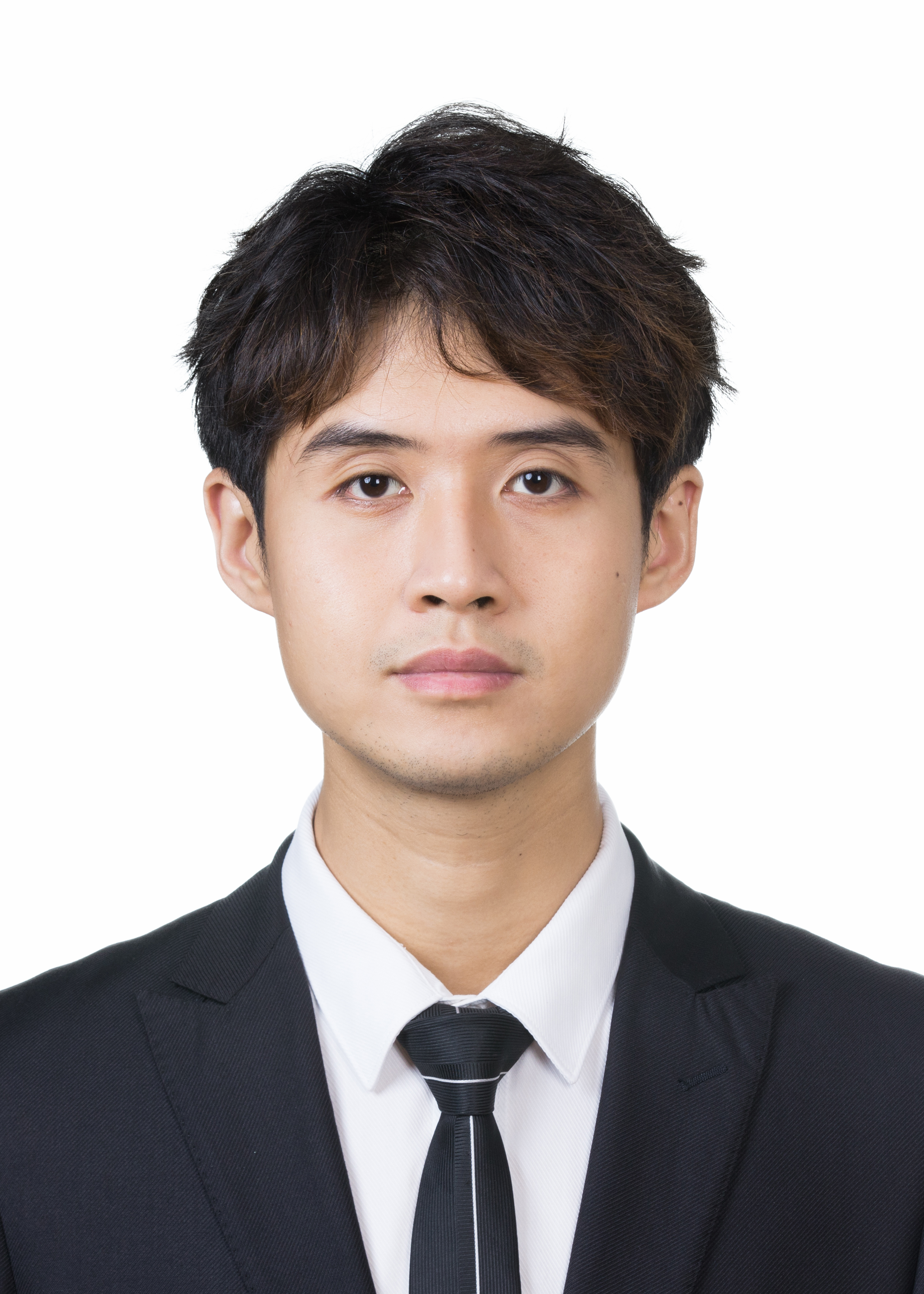}}]
{Xunkai Li} is currently working toward the PhD degree with the School of Computer Science, Beijing Institute of Technology. He received the BS degree in computer science from Shandong University in 2022. His research interests include data-centric graph intelligence, with a focus on graph databases and machine learning, multimodal understanding and generation, and AI for Science (AI4Science). He has published more than 20 papers in leading ML/DB/DM/AI conferences, including ICML, NeurIPS, VLDB, WWW, and AAAI.
\end{IEEEbiography}

\begin{IEEEbiography}[{\includegraphics
[width=1in,height=1.25in,clip,keepaspectratio]{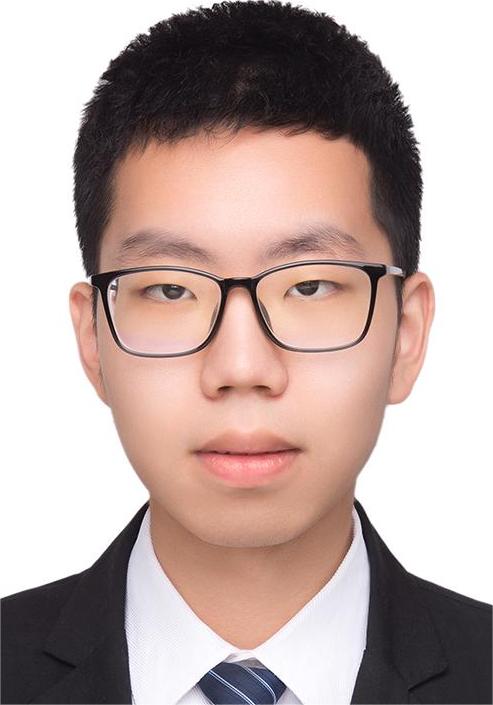}}]
{Zekai Chen} is currently pursuing his Master's degree in Computer Science at Beijing Institute of Technology under the supervision of Professor Rong-Hua Li. He obtained his Bachelor degree in Computer Science from the same institution in 2024. His research focuses on Graph Machine Learning and AI for Science (AI4Science).
\end{IEEEbiography}

\begin{IEEEbiography}[{\includegraphics
[width=1in,height=1.25in,clip,keepaspectratio]{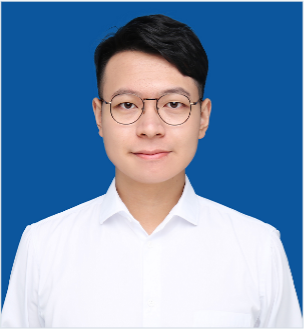}}]
{Zhengyu Wu} is currently pursuing his PhD degree in Beijing Institute of Technology (BIT), Beijing, China. His research interests include Federated Graph Learning, directed graph learning, and social network analysis. He has co-authored papers published in top ML/DB/DM/AI conferences such as ICML, VLDB, WWW, AAAI.
\end{IEEEbiography}

\begin{IEEEbiography}[{\includegraphics
[width=1in,height=1.25in,clip,keepaspectratio]{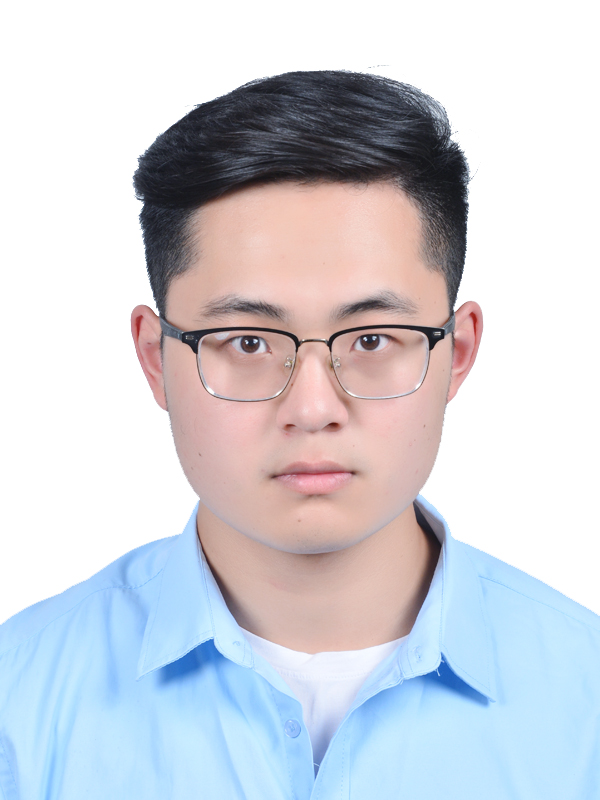}}]
{Henan Sun} received his Bachelor degree and Master degree from the College of Information Science, Northwest A\&F University in 2022, and the School of Computer Science \& Technology, Beijing Institute of Technology in 2024, respectively. His research interests include graph neural networks and graph out-of-distribution learning. He has published several academic papers in the ML community, such as ICDE, CEA and COMPSAC.
\end{IEEEbiography}

\vspace{-1em}

\begin{IEEEbiography}[{\includegraphics
[width=1in,height=1.25in,clip,keepaspectratio]{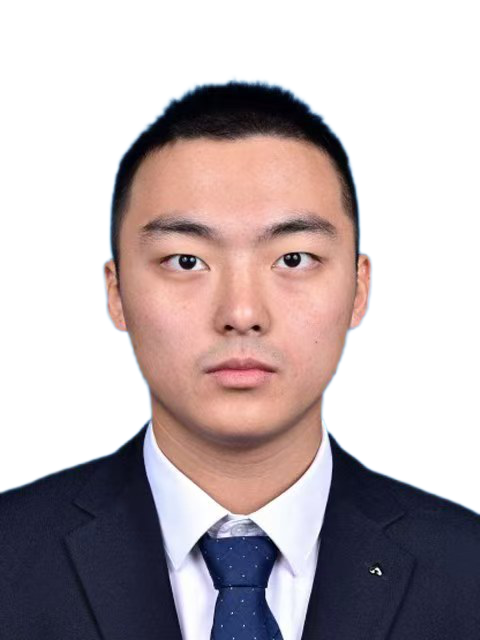}}]
{Daohan Su} is currently working toward the master's degree in the School of Computer Science \& Technology, Beijing Institute of Technology. He received his bachelor's degree from the same institution in 2024. His research interests lie in the areas of graph neural networks and large language models.
\end{IEEEbiography}
\vspace{-2em}
\begin{IEEEbiography}[{\includegraphics
[width=1in,height=1.30in,clip,keepaspectratio]{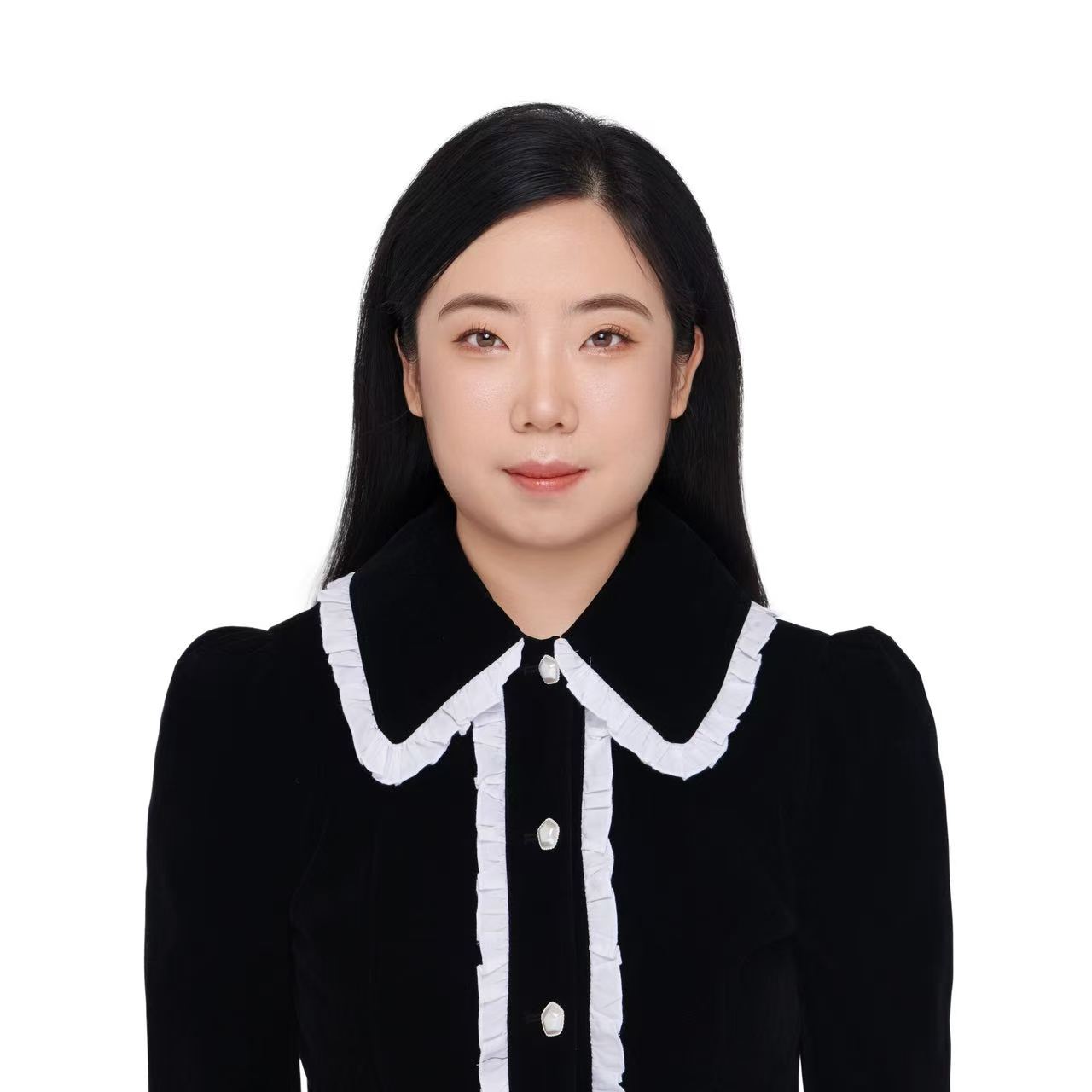}}]
{Guang Zeng} graduated from Northeastern University (Computer Science) in 2017 and currently works as an Algorithm Engineer at Ant Group. Specializing in time series forecasting, causal effect estimation, and graph computing, she has authored multiple academic papers and continues to explore innovative applications of cutting-edge technologies in real-world business scenarios.
\end{IEEEbiography}
\vspace{-2em}
\begin{IEEEbiography}[{\includegraphics
[width=1in,height=1.30in,clip,keepaspectratio]{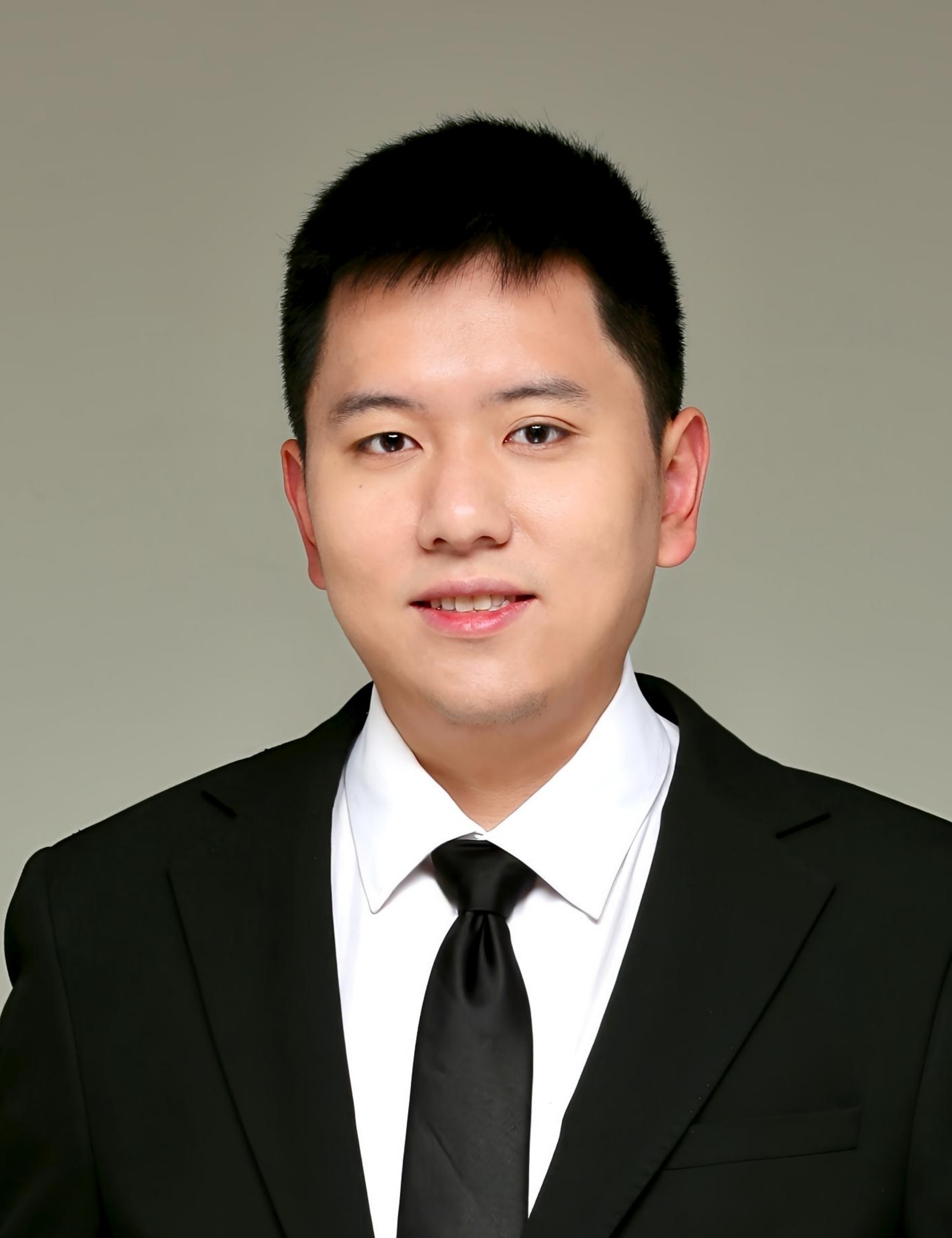}}]
{Hongchao Qin} is currently an Assistant Professor with Beijing Institute of Technology, China. He received the B.S. degree in mathematics, M.E. degree, and Ph.D. degree in computer science from Northeastern University, China in 2013, 2015, and 2020, respectively. His research interests include community detection, graph mining, knowledge graphs, financial fraud detection, vector databases, and RAG-based intelligent question answering systems.
\end{IEEEbiography}
\vspace{-2em}
\begin{IEEEbiography}[{\includegraphics
[width=1in,height=1.25in,clip,keepaspectratio]{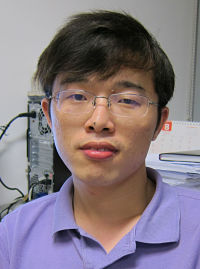}}]
{Rong-Hua Li} received his PhD degree from the Department of Systems Engineering and Engineering Management, The Chinese University of Hong Kong, in 2013, and joined the School of Computer and Software, Shenzhen University this year. He joined the School of Computer Science \& Technology, Beijing Institute of Technology in 2018 as a professor. His research interests include graph theory, graph computing systems, spectral graph theory, graph neural networks, and knowledge graphs. He is a committee member of VLDB 2024, KDD 2023, and WWW 2023.
\end{IEEEbiography}
\vspace{-2em}
\begin{IEEEbiography}[{\includegraphics
[width=1in,height=1.25in,clip,keepaspectratio]{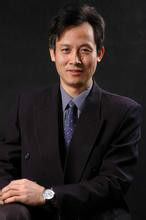}}]
{Guoren Wang} received his BS, MS, and PhD degrees from Northeastern University in 1988, 1991, and 1996. He joined the School of Computer Science \& Technology, Beijing Institute of Technology in 2018 as a professor. He is a member of the Discipline Evaluation Group of the State Council and the Expert Review Group of the Information Science Department of the National Natural Science Foundation of China. His research interests include uncertain data management, data-intensive computing, visual media data analysis, and bioinformatics.
\end{IEEEbiography}

\end{document}